\documentclass{article} 
\PassOptionsToPackage{sort&compress}{natbib}
\usepackage{iclr2027_conference,times}
\setcitestyle{numbers,square,citesep={,},aysep={,},yysep={;}}

\usepackage{amsmath,amsfonts,bm}

\def\eqref#1{equation~\ref{#1}}

\def\1{\bm{1}}

\DeclareMathAlphabet{\mathsfit}{\encodingdefault}{\sfdefault}{m}{sl}
\SetMathAlphabet{\mathsfit}{bold}{\encodingdefault}{\sfdefault}{bx}{n}

\usepackage{amsthm}
\usepackage{xcolor}
\usepackage{hyperref}
\hypersetup{
    colorlinks=true,
    linkcolor={[HTML]{D75545}},
    citecolor={[HTML]{6788ED}},
    urlcolor=blue,
    pdfborder={0 0 0},
    colorlinks=true
}
\usepackage{url}

\newtheorem{definition}{Definition}[section]
\newtheorem{assumption}{Assumption}[section]
\newtheorem{proposition}{Proposition}[section]
\newtheorem{example}{Example}[section]
\newtheorem{remark}{Remark}[section]
\newtheorem{lemma}{Lemma}[section]
\usepackage{tcolorbox}
\usepackage{booktabs}  
\usepackage{multirow}  
\usepackage{graphicx}   
\usepackage{caption}
\usepackage[table]{xcolor} 
\usepackage{amssymb}
\usepackage{algorithm}
\usepackage{algorithmic}
\usepackage{float}              
\usepackage{tikz}
\usetikzlibrary{tikzmark,fit}
\usepackage{makecell}
\usepackage[capitalize]{cleveref}
\usepackage{wrapfig}
\crefname{assumption}{Assumption}{Assumptions}
\usepackage{titletoc}
\usepackage{enumitem}
\usepackage{longtable}
\newcommand{\xhdr}[1]{{\noindent \textbf{#1}}}

\definecolor{cVerify}  {HTML}{979797}
\definecolor{cSparsify}{HTML}{DD7D77}
\definecolor{cSplit}   {HTML}{96B4FC}

\newcommand{\icmt}[1]{{\color{black!55}\;$\triangleright$\,#1}}
\newcommand{\gap}{\vspace{0.30em}}
\newcommand{\swatch}[1]{\textcolor{#1}{\rule[-0.05em]{0.7em}{0.7em}}}
\def\algw{0.435\textwidth}        

\title{What Makes a Good Symbol for Hybrid\\ Discrete–Continuous Planning?}

\author{Yiyuan Pan$^{1}$\quad Changliu Liu$^{1}$\quad Jiayuan Mao$^{2}$\\
$^{1}$Robotics Institute, Carnegie Mellon University\qquad $^{2}$University of Pennsylvania
}

\iclrfinalcopy 
\newcommand{\headerlogos}{\raisebox{1.4pt}{\includegraphics[height=12.4pt]{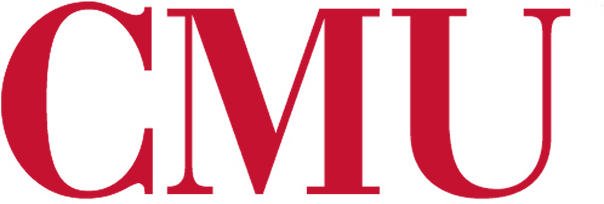}}\hspace{8pt}%
  \raisebox{1.4pt}{\includegraphics[height=12.4pt]{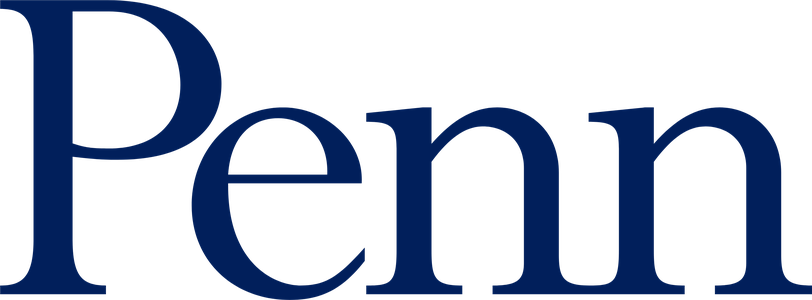}}}
\fancypagestyle{logofirst}{%
  \fancyhf{}\lhead{\headerlogos}\cfoot{\thepage}%
  \renewcommand{\headrulewidth}{0.4pt}}

\begin{document}

\maketitle
\thispagestyle{logofirst}
\pagestyle{plain}

\begin{abstract}
\vspace{-0.5em}
The physical world has continuous state and action spaces, yet robot planning often benefits from hybrid discrete-continuous formulations that explicitly expose combinatorial structure, such as contact modes or discrete primitive actions. Existing methods typically rely on human-defined symbol sets or symbols derived from language supervision, without asking what makes a symbol computationally {\it useful}. In this paper, we take an inference-algorithm--centric view to define {\it usefulness}: a good symbol should expose the discrete commitments that simplify solving the original planning problem. We first introduce the concept of \emph{solver-sufficient symbolization} by studying finite-horizon piecewise-affine systems, for which solver sufficiency can be precisely characterized, and develop a counterexample-guided algorithm for symbol invention. We then extend this principle to invent symbols for optimization objectives parameterized by general neural networks as energy functions, showing that solver-sufficient symbols can nevertheless be induced from their continuous optimization landscapes. Experiments on simulated and real-robot tasks show that the resulting symbols reduce solving cost while preserving solution quality, while also improving interpretability and steerability. Furthermore, our theoretical and algorithmic insights help explain why certain human-defined symbols are computationally useful.

\end{abstract}

\vspace{-15pt}

\section{Introduction}\label{sec:intro}
\vspace{-0.5em}
The physical world has continuous state and action spaces, yet planning often benefits from hybrid discrete--continuous formulations that explicitly expose combinatorial structure, such as contact modes in manipulation (e.g., {\it fixed and sliding contact}s)~\citep{posa2016optimization,cheng2022contact} and symbolic relational states (e.g., {\it holding}, {\it in})~\citep{garrett2021integrated}. From a resource-rational perspective in cognitive science, a useful symbol can expose key commitments and abstract away details to support efficient planning under bounded computation \citep{lieder2020resource,ho2022people}: When a person wants to go to the office with a bag, the motion of the person and the bag must initially be considered jointly. Once ``the bag is being carried'' is established as a symbolic commitment, the bag’s motion can be abstracted away when planning the person’s path. 

While humans can invent symbols to reduce the complexity of decision-making, deep learning--based solutions to decision-making often choose to encode the full decision landscape within a single continuous-space model~\citep{chi2025diffusion,janner2022planning}. They remove the need to declare discrete modes and dependencies {\it a priori}, but fundamentally cannot remove the underlying combinatorial structure of the problem. For example, for neural-network--parameterized dynamics models, basins and barriers appear in the optimization landscape. This makes planning or optimization hard as tasks and horizons grow \citep{florence2022implicit}.
To this end, researchers have considered exposing symbol commitments explicitly as state variables, such as the object relation of person-carrying-bag, to break much of the global coupling in planning.
The symbol invention problem studies inducing these symbols from data, but typically from predefined grammars and labels \citep{silver2023predicate,athalye2025predicate}. Such symbols are often interpretable and descriptive, yet their construction typically relies on heuristics about the problem structure, such as sparsity and locality. As a result, they may hinder rather than simplify planning. In this paper, we address this fundamental question:
\begin{tcolorbox}[top=1pt, bottom=1pt, left=0pt, right=0pt]
\begin{center}
\textit{What makes a symbol {\bf useful} in hybrid discrete-continuous planning?}
\end{center}
\end{tcolorbox}
Our definition of a useful symbol comes from the perspective of inference-time solving. We consider a general family of hybrid discrete-continuous planning algorithms based on \emph{search-then-optimize} \citep{geoffrion1972generalized,dechter1990enhancement}: symbols partition the entire decision space into conditioned subproblems, within which continuous-space optimization is performed. Let $x\in\mathcal X$ denote all decision variables and $h$ the planning objective. Our goal is to solve the hybrid optimization problem:
\begin{equation}
    \mathcal{P}_{\mathrm{joint}}:\
    \min_{x\in\mathcal X} h(x).
    \label{eq:general_joint_problem}
\end{equation}
Let $\Phi:\mathcal X\rightarrow\Psi$ denote a symbolization map, where each $\psi\in\Psi$ is a symbolic assignment. Fixing $\psi$ induces the conditioned problem $\mathcal{P}_{\mathrm{joint}}^{\psi}:  
\min\nolimits_{x\in\Phi^{-1}(\psi)} h(x|\psi)$. This yields the decomposition:
\begin{equation}
    \underbrace{\operatorname{Solve}(\mathcal{P}_{\mathrm{joint}})}_{\text{\textit{Joint Solving}}}
    =
    \underbrace{\min\nolimits_{\psi\in\Psi}}_{\text{\textit{Symbolic Search}}}
    \underbrace{\operatorname{Solve}(\mathcal{P}_{\mathrm{joint}}^{\psi})}_{\text{\textit{Grounded Solving}}}.
    \label{eq:symbol_guided_decomposition}
\end{equation}
whose na\"ive solving complexity is $\mathcal{O}(2^{|\Psi|}\cdot {\rm SolveTime(\mathcal{P}^\psi_{\rm joint})})$. In this paper, we focus on \emph{solver-sufficient} symbols: they preserve the global optimum in \cref{eq:symbol_guided_decomposition} while reducing the overall complexity of conditioned solve ${\operatorname{Solve}(\mathcal{P}_{\mathrm{joint}}^{\psi})}$. In a planning context, such symbols should decouple temporal and spatial dependencies that would otherwise require joint reasoning, allowing each grounded solve to involve only the local variables relevant to the selected symbolic commitment.


To put these in practice, we begin with simple piecewise affine (PWA) systems \citep{rodrigues2003observer}, where the solver sufficiency and complexity can be studied exactly using TU-dimension tools \citep{hoffman2010integral}. We then extend this perspective to general settings involving neural networks, focusing on energy-based formulations which include frameworks such as diffusion and flow-matching policies \citep{chi2025diffusion,lipman2022flow}. There, solver sufficiency can only be verified on samples due to the complexity of modern networks. In both settings, we develop algorithms for both symbol invention and symbol-guided planning. Our contributions are threefold:

\begin{itemize}[leftmargin=1.2em, itemsep=2pt, parsep=0pt, topsep=2pt]
\item We introduce solver-sufficient symbolization, in which symbolic commitments preserve optimality while reducing grounded solving cost. This formalizes the symbolic invention problem as an optimization problem rather than a purely descriptive interface.

\item We provide a mathematical foundation for analyzing why certain human-defined concepts are useful in robotic planning, and characterize task hardness by the amount of irreducible discrete structure that must remain explicit for faithful solving.

\item We develop and evaluate a symbol invention and hybrid planning framework that retains key advantages of symbolic planning, including interpretability and behavior steerability.
\end{itemize}

\paragraph{Overview of related work.}
In search-then-optimize planning \citep{geoffrion1972generalized,audet2026partitioned}, the key step is to choose a conditioned subproblem, such as mode assignments in MILPs \citep{marcucci2020warm,cauligi2021coco} and energy-based model inference with multi-seed initialization \citep{florence2022implicit,chi2025diffusion}. This choice is hard because decisions are coupled. Researchers have formalized such coupling in AI planning problems~\citep{chen2007act,mao2023planning}. Hybrid discrete-continuous planners aim to address this by introducing symbols \citep{konidaris2018skills,garrett2020pddlstream}, but their symbols come from labels or simple algebraic combinations of state variables \citep{silver2023predicate,athalye2025predicate}. We instead derive symbols from the solver itself: the commitments required before make the conditioned problem cheap to solve, and the commitment number measures the problem's coupling and irreducible combinatorial hardness. We include a more comprehensive discussion of related work in Appendix~\ref{app:related_works}.


\vspace{-0.5em}
\section{Solver Sufficiency in PWA Systems}
\vspace{-0.5em}

\begin{figure*}
    \centering
    \includegraphics[width=\linewidth]{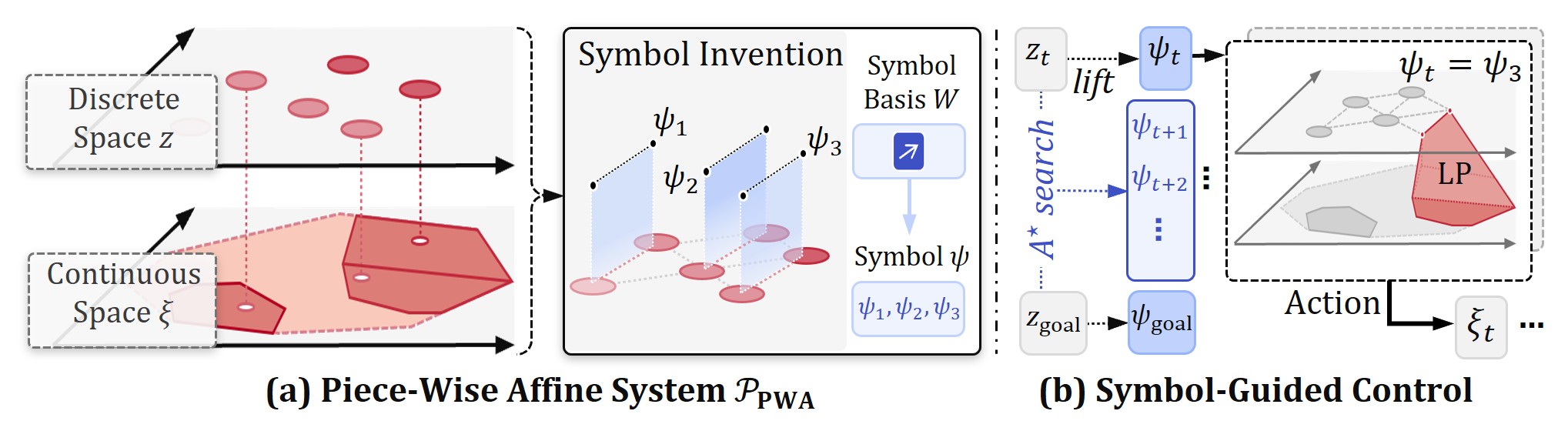}
    \vspace{-20pt}
    \caption{\textbf{Solver-Sufficient Symbolization for PWA Systems.} Relative affine-TU groups optimization regions with a compact symbol set. Each conditional optimization subproblem $\mathcal{P}_{\mathrm{joint}}^{\psi}$, given by an assignment of discrete symbols $\psi$, reduces to an LP-relaxable optimization problem.}
    \label{fig:overview_pwa}
    \vspace{-10pt}
\end{figure*}

We first build the intuition for solver-sufficient symbolization in piecewise affine (PWA) systems. In PWA, a finite-horizon decision-making problem is formulated as a mixed-integer linear program (MILP). Here, discrete variables encode {\it modes}, such as contacts, while continuous variables describe the continuous states and actions governed by mode-dependent affine dynamics. Such formulations are widely used in contact-rich manipulation, such as linear complementarity models~\citep[][\textit{inter alia}]{jin2024task}. Formally, let $z\in\{0,1\}^{n_d}$ denote discrete mode assignments and $\xi\in\mathbb{R}^{n_c}$ denote continuous states and actions. With $x=(z,\xi)$, \cref{eq:general_joint_problem} specializes to
\begin{equation}
\begin{aligned}
\mathcal{P}_{\mathrm{PWA}}:\quad
\min\nolimits_{z,\xi}\quad
& c_d^\top z+c_c^\top \xi\\
\mathrm{s.t.}\quad
& E^d z=e^d,\ P^d z\le p^d;\quad E^c \xi=e^c,\ P^c \xi\le p^c,\\
& B^d z+B^c \xi\le b; \quad z\in\{0,1\}^{n_d},\ \xi\in\mathbb{R}^{n_c}.
\end{aligned}
\label{eq:pwa_joint_problem}
\end{equation}
The first set of constraints describes logical constraints over modes and linear constraints over continuous variables. The last constraint couples mode assignments with continuous states and actions. Fixing all variables in $z$ leaves a simple linear program (LP) problem, but searching over these assignments is combinatorially expensive. Solver-sufficient symbolization seeks a more compact discrete commitment set that yields better tractability.

\vspace{-0.5em}
\subsection{Affine-TU Symbolization in Integer Programs}
\vspace{-0.5em}
A classical source of tractability in integer programming (IP) is \emph{total unimodularity} (TU). It is a structural condition under which integrality constraints can be removed without affecting optimality.
\begin{definition}[Total Unimodularity]
\label{theo:hoffman}
A matrix $A\in\mathbb{Z}^{m\times n}$ is
\emph{totally unimodular} if every square submatrix of $A$
has determinant in $\{-1,0,1\}$.
\end{definition}
For an IP $\{x\in \mathbb{Z}^n_{\geq 0}:Ax\leq b\}$, by the Hoffman--Kruskal theorem~\citep{hoffman2010integral},
$A$ is TU if and only if, for every $b\in\mathbb{Z}^m$, the polyhedron $\{x\in\mathbb{R}_{\ge 0}^n:Ax\le b\}$ is integral. TU therefore certifies that the corresponding integer program can be solved through its LP relaxation (by removing the integer constraints on $x$). When the original matrix $A$ is not TU, we can create TU subproblems by fixing a few ``summaries'' of the decision variables. \emph{Affine TU} formalizes this:

\begin{definition}[Affine TU~\citep{bader2018mixed}]
Let $A\in\mathbb{Z}^{m\times n}$.
A $k$-row \emph{affine TU decomposition} of $A$ consists of
matrices
$
\widetilde A\in\mathbb{Z}^{m\times n},\
U\in\mathbb{Z}^{m\times k},\
W\in\{0,\pm1\}^{k\times n}
$
such that
\[
A=\widetilde A+UW,
\qquad
[\widetilde A;W]\ \text{is TU} \footnote{$[A;B]$ denotes vertical concatenation. The same applies thereafter.}.
\]
The \emph{affine TU dimension} of $A$ is the minimum $k$ for which such a decomposition exists.
\end{definition}
To see a connection to symbol invention, we can view the discrete summary $\psi=Wx\in\mathbb{Z}^k$ as the assignment to a set of {\it new} symbols induced by $W$. For a fixed assignment $\psi$, the original constraints $Ax\le b$ becomes $\widetilde A x\le b-U\psi$. Since $[\widetilde A;W]$ is TU, this conditioned IP problem has an exact LP relaxation. As a result, instead of searching for the integral assignment of $x$, we can search over the integral assignments for $\psi$ and solve LPs given each assignment. We get a speed-up when $k \ll n$. This is precisely the principle of solver-sufficient symbolization: introduce discrete symbols that capture the critical decisions needed to make the remaining optimization problem easy. 

\vspace{-0.5em}
\subsection{Relative Affine-TU Symbolization in MILP}
\vspace{-0.5em}
Affine TU applies to pure integer programs.
However, for \cref{eq:pwa_joint_problem}, we additionally need to account for the coupling between discrete modes and continuous trajectories. The equalities $E^d z=e^d$ already fix certain summaries of $z$, so the symbols to introduce need only capture information beyond those equalities ({\it relative} to $E^d$). This motivates the following extension:\footnote{We assume $E^d$ is TU and $B^d\in\{0,\pm1\}^{m_B\times n_d}$. These assumptions accommodate common PWA systems in which discrete symbols enter through mode selection and signed constraints. See discussion in Appendix~\ref{app:assumptions}.
}

\begin{definition}[Relative Affine-TU]
\label{def:rat_certificate}
A matrix $W\in\{0,\pm1\}^{k\times n_d}$ is a
\emph{relative affine-TU certificate} for the discrete interface
$(E^d,P^d,B^d)$ if there exist integer matrices
$\widetilde P^d,U_P,V_P$ and rational matrices $U_B,V_B$
such that
\[
    (1) P^d = \widetilde P^d+U_PW+V_PE^d, \quad (2) B^d = U_BW+V_BE^d, \quad (3) M := [\widetilde P^d;E^d;W]\ \text{is TU}.
\]
The \emph{relative affine-TU dimension}
$\kappa(\mathcal{P}_{\rm joint})$ is the minimum number of rows $k$ among such certificates. Such decomposition always exists, since $W$ can take $\mathbf{I}_{n_d\times n_d}$. Thus, $\kappa(\mathcal{P}_{\rm joint})\leq n_d$.
\end{definition}

To gain some insight, similarly consider fixing a symbol assignment $\psi=Wz$: as illustrated in \cref{fig:overview_pwa} (a), the map $W$ groups primitive mode assignments that share the same symbolic value, depicted by the blue slices labeled $\psi_1,\psi_2,\psi_3$. Fixing $\psi$ therefore corresponds to choosing a particular group of assignments of all combinatorial assignments of $z$. Within this group, the discrete variable $z$ can be relaxed to continuous and the inequalities reduce to $\widetilde P^d z \le p^d-U_P\psi-V_Pe^d$, which is an integral polytope. That is, the concrete mode decisions of the full $z$ can be recovered after a linear optimization. Meanwhile, the coupling constraints reduce to $B^c\xi\le b-U_B\psi-V_Be^d,$ which depend only on the selected symbols. 
This is the mechanism underlying the symbol-guided control in \cref{fig:overview_pwa} (b): a planner searches over symbolic commitments $\psi$ and, once a commitment is fixed (e.g., as $\psi_3$ in the figure), solves the corresponding residual problem exactly through an LP relaxation. $W$ is computed once per task since it is independent of the start and the goal state.

Fundamentally, the relative affine-TU dimension $\kappa(\mathcal{P}_{\rm joint})$ quantifies the irreducible discrete variables, implying the problem hardness is characterized by at most $2^{\kappa(\mathcal{P}_{\rm joint})}$ modes. Such gains can be quantified in terms of formal complexity statements in Appendix~\ref{app:rat_complexity}. Furthermore, the minimum number of symbol coordinates is bounded, as given by the following lemma.
\begin{lemma}[Bounds on Relative Affine-TU Dimension]
\label{prop:rat_bounds}
Under the assumptions above, $
\operatorname{rank}[B^d;E^d]
-\operatorname{rank}(E^d)
\;\le\;
\kappa(\mathcal{P}_{\rm joint})
\;\le\;
n_d.
$ Proofs are in Appendix~\ref{app:rat_bounds}.
\end{lemma}

\vspace{-0.5em}
\subsection{Counterexample-Guided Symbol Invention with CPR}
\vspace{-0.5em}
Determining the minimum affine-TU dimension is NP-hard. We therefore construct a certificate $W$ using \emph{Counterexample-Guided Partition Refinement} (CPR), a CEGAR-style procedure summarized in \cref{alg:pwa_cpr}. CPR initializes $W_0$ from directions in $B^d$ not represented by $E^d$, guided by the rank lower bound in \cref{prop:rat_bounds}. Given the current $W_i$, it residualizes $P^d$ against $[W_i;E^d]$ to obtain $\tilde P_i^d$ and forms the verification matrix $M_i=[\tilde P_i^d;E^d;W_i]$ in \cref{def:rat_certificate}. The TU oracle \textsc{IsTU}$(M_i)$ then determines whether $W_i$ satisfies the relative affine-TU certificate; if so, CPR terminates.

When verification fails, the cause may lie not in which symbols $W_i$ encodes, but in how it is represented as a matrix. 
The \textsc{Sparsify} step rewrites $W_i$ with another $\{0,\pm1\}$ basis of the same row space.
If no such re-basing yields a TU matrix, the obstruction is real. 
The \textsc{Split} step then shrinks $M_i$: we delete rows and columns until obtaining a minimal non-TU submatrix $O_i$, in which removing any single column makes the matrix TU. 
CPR therefore promotes the variable $z_c$ of one such column to a new symbol ($k\rightarrow k+1$): once a value is assigned to $z_c$, its column drops out of the conditioned problem, so the remaining columns of $O_i$ form a TU matrix. The loop then repeats. We include more algorithm and implementation details in Appendix~\ref{app:pwa_cpr}.


\begin{figure}[t]
\centering
\begin{minipage}[t]{0.49\textwidth}
\begin{algorithm}[H]
\footnotesize
\caption{CPR Algorithm (PWA System)}\label{alg:pwa_cpr}
\textbf{Init:} PWA system $(E^{d},P^{d},B^{d})$; TU oracle \textsc{IsTU}\\
\textbf{Certificate:} Relative-affine TU $W$
\vspace{2pt}\hrule\vspace{4pt}
\begin{algorithmic}[1]
\STATE $W_0 \gets B^{d}\bmod\mathrm{rowspan}(E^{d})$\icmt{\cref{prop:rat_bounds}}\gap
\REPEAT\gap
\STATE \tikzmarknode{L3}{$\tilde P^{d}_i \gets P^{d}\bmod[W_i;E^{d}]$}\gap
\STATE \tikzmarknode{L4}{$M_i \gets [\,\tilde P^{d}_i;\;E^{d};\;W_i\,]$}\gap
\STATE \tikzmarknode{L5}{\textbf{if} \textsc{IsTU}$(M_i)$, \textbf{return} $W_i$}\gap
\STATE \tikzmarknode{L6}{re-base $W_i$ in $\mathrm{rowspan}(W_i)$}\icmt{\textsc{Sparsify}}\gap
\STATE \tikzmarknode{L7}{refresh $M_i$; \textbf{if} \textsc{IsTU}$(M_i)$, \textbf{return} $W_i$}\gap
\STATE \tikzmarknode{L8}{$O_i\gets$ minimal non-TU core of $M_i$}\gap
\STATE \tikzmarknode{L9}{$W^{c}\gets \textsc{Split}(W_i,O_i,c),\quad \forall c\in\mathrm{cols}(O_i)$}\gap
\STATE \tikzmarknode{L10}{$W_{i+1}\gets \textsc{Select}\!\left(\{W^{c}\}_{c};\,\textsc{IsTU}\right)$}\\
\icmt{New symbol invented, $k\!\to\!k{+}1$}\gap
\UNTIL{$k=n_{d}$}\gap
\STATE \textbf{Return} $W$
\end{algorithmic}
\end{algorithm}
\end{minipage}%
\hfill%
\begin{minipage}[t]{0.48\textwidth}
\begin{algorithm}[H]
\footnotesize
\caption{CPR Algorithm (NN System)}\label{alg:neural_cpr}
\textbf{Init:} Network system $h_s(\cdot)$; resolution $\tau$\\
\textbf{Certificate:} $\tau$-descent faithfulness $W:=(\mathcal{U},\mathcal{Q})$
\vspace{2pt}\hrule\vspace{4pt}
\begin{algorithmic}[1]
\STATE $(\mathcal{U}_0,\mathcal{Q}_0)\gets\emptyset$;\; $G_{s}\gets$ descent samples\icmt{$k=0$}\gap
\REPEAT\gap
\STATE \tikzmarknode{R3}{partition $G_s$ into subproblems by $(\mathcal{U}_i,\mathcal{Q}_i)$}\gap
\STATE \tikzmarknode{R4}{$\Phi_i\gets$ subproblems that are not $\tau$-faithful}\gap
\STATE \tikzmarknode{R5}{\textbf{if} $\Phi_i=\emptyset$, \textbf{return} $(\mathcal{U}_i,\mathcal{Q}_i)$}\gap
\STATE \tikzmarknode{R6}{$G_{s}\gets G_{s}\cup$ samples near $\Phi_i$}\gap
\STATE \tikzmarknode{R7}{recompute $\Phi_i$; \textbf{return} $(\mathcal{U}_i,\mathcal{Q}_i)$ \textbf{if} $\Phi_i=\emptyset$}\gap
\STATE \tikzmarknode{R8}{$\pi\gets$ lowest pass between two basins in $\Phi_i$}\gap
\STATE \tikzmarknode{R9}{$u\gets\nabla h_s$ flip across $\pi$;\; $q\gets$ cut at $u^{\top}\!x_{\pi}$}\gap
\STATE \tikzmarknode{R10}{$\mathcal{U}_{i+1}\!\gets\!\mathcal{U}_i\!\cup\!\{u\}$;\; $\mathcal{Q}_{i+1}\!\gets\!\mathcal{Q}_i\!\cup\!\{q\}$}\\
\icmt{New symbol invented, $k\!\to\!k{+}1$}\gap
\UNTIL{iteration limit reached}\gap
\STATE \textbf{Return} $W:=(\mathcal{U},\mathcal{Q})$
\end{algorithmic}
\end{algorithm}
\end{minipage}
\begin{tikzpicture}[remember picture, overlay,
  bx/.style = {draw=#1, densely dotted, line width=1pt, rounded corners=2pt,
               fill=none, inner xsep=1pt, inner ysep=0pt},
  bx/.default = black!55]
\node[bx=cVerify,   fit={(L3)(L4)(L5)([xshift=\algw]L3.west)}]   {};
\node[bx=cSparsify, fit={(L6)(L7)([xshift=\algw]L6.west)}]       {};
\node[bx=cSplit,    fit={(L8)(L9)(L10)([xshift=\algw]L8.west)}]  {};
\node[bx=cVerify,   fit={(R3)(R4)(R5)([xshift=\algw]R3.west)}]   {};
\node[bx=cSparsify, fit={(R6)(R7)([xshift=\algw]R6.west)}]       {};
\node[bx=cSplit,    fit={(R8)(R9)(R10)([xshift=\algw]R8.west)}]  {};
\end{tikzpicture}
\caption{\textbf{One CEGAR Skeleton, Two Instantiations.} Both algorithms share three stages: \swatch{cVerify}~\textsc{Verify}, \swatch{cSparsify}~\textsc{Sparsify} (\textit{eliminate false obstacles}), and \swatch{cSplit}~\textsc{Split} (\textit{pay one symbol for real obstacles}).}
\label{fig:cpr_pair}
\vspace{-10pt}
\end{figure}

\vspace{-0.5em}
\section{Solver Sufficiency in Network--Parameterized Systems}
\vspace{-0.5em}
The PWA analysis illustrates the central principle of solver-sufficient symbolization: identify discrete commitments that make the remaining optimization tractable. There, the discrete choices are explicit mode variables, and relative affine-TU certifies that fixing their symbolic summaries leaves an LP-solvable problem. We now consider decision making with neural objectives, where the same principle is useful but the characterization of difficulty is different, due to the black-box nature and large parameterization of modern networks.

\xhdr{Energy-based problem formulation.}
Let $s$ denote the current state or planning context, and let $\mathcal{X}=[\underline{x},\overline{x}] \subset\mathbb{R}^{n}$ be a bounded decision space. We represent decision-making as minimizing a state-conditioned neural energy function:
\begin{equation}
\mathcal{P}_{\mathrm{NN}}:\quad
\min\nolimits_{x\in\mathcal{X}} h_s(x).
\label{eq:neural_problem}
\end{equation}
Here, energy denotes a scalar measure of the decision quality, with lower values indicating more desirable decisions; it need not represent physical energy. For single-step control, $x=a$ and one may use the Q function $h_s(a)=-Q(s,a)$. For long-horizon planning, $x=a_{1:H}$ and $h_s$ can evaluate the trajectory through a learned trajectory evaluator or a world model. Unlike the conditioned LPs in the PWA setting, this energy landscape remains highly nonconvex. As illustrated on the left of \cref{fig:overview_nn} (a), doing gradient descent or Langevin dynamics from different initializations may lead to different local minima. Our goal here is to characterize {\it regions} in such optimization landscapes.

\begin{remark}
     Our extension removes the discrete variables $z$ in the base optimization formulation. In PWA, the raw model is affine within each mode, and the complexity of the global landscape comes from $z$. When generalizing this to neural network--parameterized landscapes, we can represent such a complex landscape directly over $x$, without declaring mode variables, as illustrated in \cref{fig:overview_nn} (a).
\end{remark}

\begin{figure*}
    \centering
    \includegraphics[width=\linewidth]{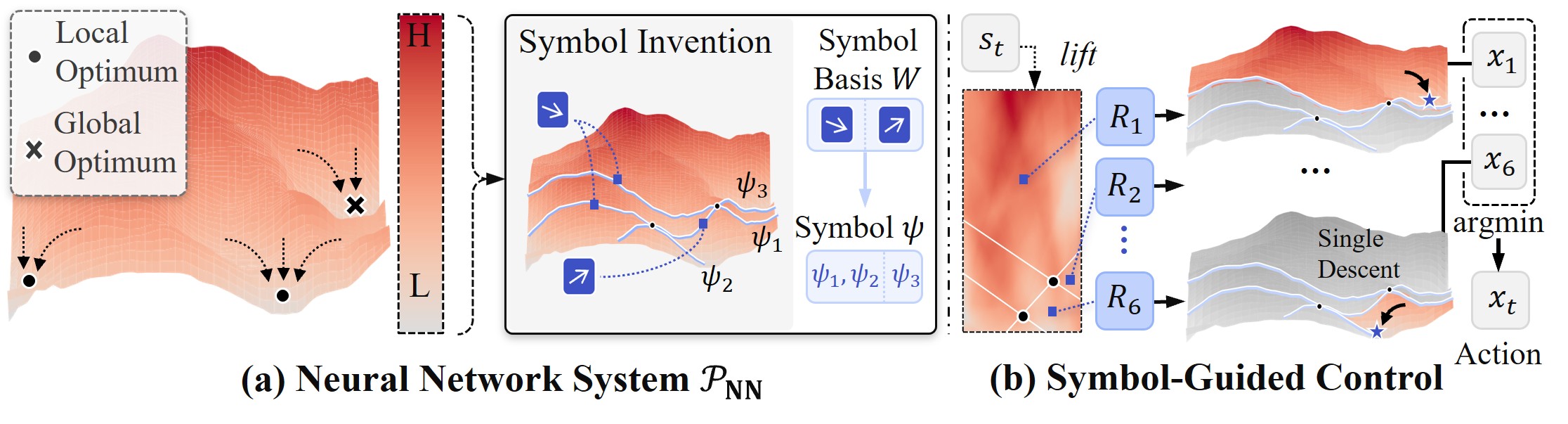}
    \vspace{-20pt}
    \caption{\textbf{Solver-Sufficient Symbolization for NN Systems.} $\tau$-descent faithfulness partitions a multimodal state-conditioned objective into regions. A symbol-guided control algorithm performs one confined descent per region and selects the lowest-value candidate across all regions.}
    \label{fig:overview_nn}
    \vspace{-15pt}
\end{figure*}


\vspace{-0.5em}
\subsection{$\mathbf{\tau}$-Descent Faithful Symbolization in Neural Objectives}
\vspace{-0.5em}
In the PWA setting, the tractability of the conditioned problem follows from an algebraic property of the constraint matrix. For neural objectives,  conditioned solvability is instead characterized by the geometry that a local solver must navigate. In particular, a region may contain several local minima yet remain manageable if the solver can escape the barriers separating them. We model this capability through an escape budget $\tau>0$, which bounds the temporary increase in energy above the lowest value reached so far. A path satisfying this budget is defined as follows.

\begin{definition}[$\tau$-Admissible Path]
\label{def:tau_admissible_path}
Let $R\subseteq\mathcal{X}$ and $x\in R$.
A continuous path $\gamma:[0,1]\to R$ with $\gamma(0)=x$
is \emph{$\tau$-admissible} for $h_s$ if
\[
h_s(\gamma(t))
\le
\min_{0\le r\le t} h_s(\gamma(r))+\tau,
\qquad \forall t\in[0,1].
\]
We denote by $\Gamma^\tau_{s,R}(x)$ the set of all such
paths starting from $x$.
\end{definition}

We then define a region to be easily \textit{solvable} if this allowance suffices to reach a
near-optimal point from every initialization in a region.
Assume that $h_s$ is continuous on $\mathcal{X}$.

\begin{definition}[$\tau$-Descent Faithfulness]
\label{def:tau_univalence}
For a nonempty region $R\subseteq\mathcal{X}$ and $x\in R$,
define the best reachable endpoint value
\[
v^\tau_{s,R}(x)
:=
\inf_{\gamma\in\Gamma^\tau_{s,R}(x)}
h_s(\gamma(1)).
\]
The region $R$ is \emph{$\tau$-descent faithful} with
respect to $h_s$ if
\[
\sup_{x\in R} v^\tau_{s,R}(x)
\le
\inf_{y\in R} h_s(y)+\tau.
\]
\end{definition}


The condition does not require convexity or a unique local minimum. It certifies accessibility under the specified path model. This provides an analogy to the relative affine-TU certificate. In both settings, a symbol induces conditioned regional problems under a hardness measurement. For PWA systems, that guarantee is the exact optimizability through an LP relaxation; for neural objectives, it is the near-optimal accessibility within an escape budget.

\begin{remark}[Scope of the Guarantee]\label{rem:scope}
We study \emph{how to partition} a problem, not \emph{how to solve each part}: the solver and its capability are given (\cref{eq:symbol_guided_decomposition}). Our argument has two links. \emph{First}, $\tau$-descent faithfulness gives a good solution: in a faithful region, a $\tau$-admissible path leads from every start to within $\tau$ of the optimum (\cref{prop:tau_guarantee}). \emph{Second}, sampling gives faithfulness only on the samples, since a large network cannot be verified exactly. The first link shows that such a path exists, not that the optimizer follows it; this is an assumption on the solver. It is nevertheless a standard capability: given enough steps, an optimizer that can climb by $\tau$ is not trapped in basins shallower than $\tau$, as one can use methods such as simulated annealing \citep{hajek1988cooling} and projected Langevin algorithms \citep{lamperski2021projected}. Such methods act on the solver side, which our method takes
as given.
\end{remark}

Given the definition, we next seek symbols that induce such conditioned regions. As illustrated in \cref{fig:overview_nn} (a), 
the symbol invention problem is to choose projection directions and quantizers that organize
the decision space into regions that are $\tau$-descent faithful, with as few directions and
thresholds as possible.
Formally, we seek a fixed symbolic map
\begin{equation}
W(x)
=
\left(
q_1(u_1^\top x),\ldots,q_k(u_k^\top x)
\right),
\label{eq:neural_symbol_map}
\end{equation}
where $u_1,\ldots,u_k\in\mathcal{U}$ are projection directions and each $q_j\in\mathcal{Q}$ is a finite-valued quantizer. Each symbolic assignment $\psi$ defines a region $R_\psi:=\{x\in\mathcal{X}:W(x)=\psi\}$.

During execution, the current state $s_t$ instantiates $h_{s_t}$ and the corresponding quantizers in $W=(\mathcal{U},\mathcal{Q})$, whose assignments partition $\mathcal{X}$ into conditioned regions. The controller then performs single confined descent in each region under the $\tau$-admissible solver model and outputs the action sequence attaining the lowest objective value. By \cref{prop:tau_guarantee} and under the solver assumption of \cref{rem:scope}, the value is within $\tau$ of the global optimum. The procedure is illustrated in \cref{fig:overview_nn} (b). 

\begin{figure*}[t]
    \centering
    \includegraphics[width=0.96\linewidth]{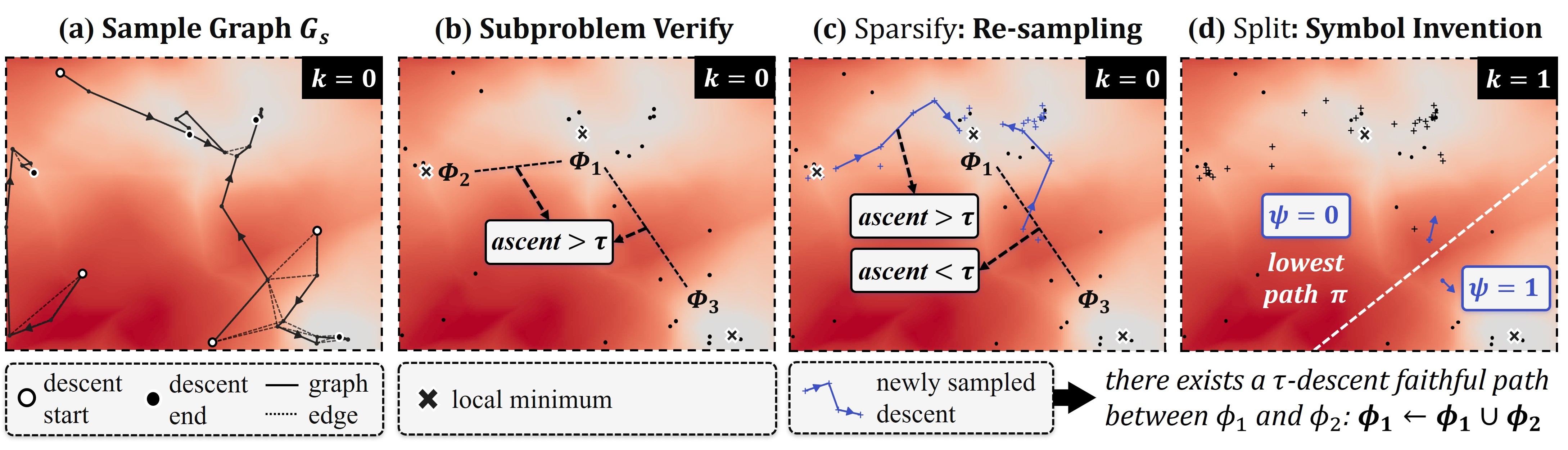}
    \vspace{-8pt}
    \caption{\textbf{Exemplar Symbol Invention in One Round.} Take Lunar Lander as the exemplar task.}
    \label{fig:example}
    \vspace{-10pt}
\end{figure*}

\vspace{-0.5em}
\subsection{Counterexample-Guided Symbol Invention with Neural CPR}
\vspace{-0.5em}
Neural CPR inherits the CEGAR style. Since the network parameters cannot be characterized explicitly, CPR constructs symbols from sampling.  Given a state set $\mathcal S$ (e.g., from demonstration trajectories), we run multi-start descent to collect outputs, objective values, and gradients.
For each state $s \in \mathcal S$, an $m$-nearest-neighbor graph $G_s$ connects these samples; each edge stores the largest objective value at $n_e$ points sampled along its segment. Next, neural CPR discovers a symbolic partition through sampling-based, counterexample-guided refinement (see \cref{alg:neural_cpr}; detailed version in Appendix~\ref{app:cpr_nn}). 
Verification is therefore exact on the sample graph $G_s$, not on the continuous region.

\cref{fig:example} gives a concrete example of a refinement loop on state $s$. \textbf{(a)} Random restarts of descent settle in different basins; we link descent nodes by $m$-nearest neighbours into a graph $G_s$.  
\textbf{(b)} Verification partitions $G_s$ into basins that are each internally $\tau$-descent faithful, i.e., $\Phi_1,\Phi_2,\Phi_3$. 
\textbf{(c)} \textsc{Sparsify}-step: such a separation may itself be an artefact of under-sampling: re-sampling descents between two basins can recover a $\tau$-admissible path connecting them, merging $\phi_1$ and $\phi_2$ into a single basin at no symbolic cost. 
\textbf{(d)} \textsc{Split}-steps: basins that remain separated lead to new symbols. We first estimate a lowest-barrier path $\pi$ from the trapped basin toward a lower-energy basin and locate a transition point $x_\pi$ where the gradient directions ``\textit{flip}'' within its neighborhood. The gradient directions provide a heuristic separating direction $u_{k+1}$. Its projection $b_{k+1}=u_{k+1}^{\top}x_\pi$ defines a candidate quantization threshold. CPR then augments the symbolic map with this new symbol. Steps b-d will be repeated until no violations are found. Details are deferred to Appendix~\ref{app:cpr_nn}.

\begin{figure*}[t]
    \centering
    \includegraphics[width=\linewidth]{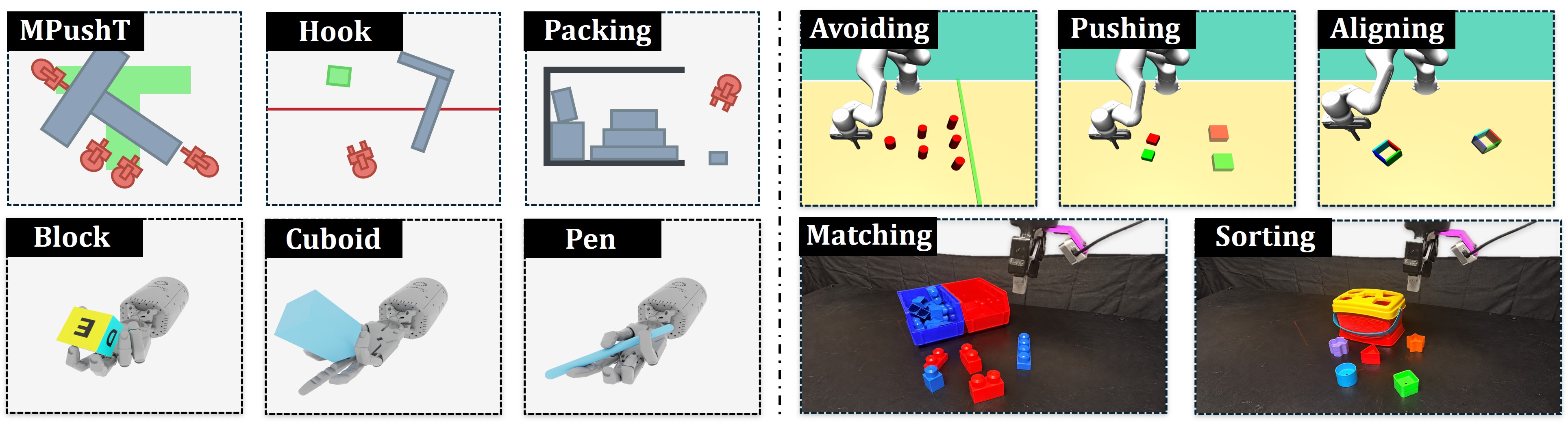}
    \vspace{-18pt}
    \caption{\textbf{Task Suites.} Dashed line divides the tasks for PWA systems and for network systems.}
    \label{fig:suites}
    \vspace{-8pt}
\end{figure*}

\vspace{-0.5em}
\section{Experiments}
\vspace{-0.5em}
In this section, we present extensive experimental results to answer the following questions: \\
\textbf{Q1.} Can solver-sufficient symbols preserve optimality while reducing solving cost? \\
\textbf{Q2.} Can solver-sufficient symbols explain why certain human-defined symbols are useful? \\
\textbf{Q3.} Can solver-sufficient symbols retain the interpretability and steerability of symbolic planning? 
\begin{table}[t]
\centering
\caption{\textbf{Solver Sufficiency on PWA Systems.}
For inference-time speedup \textbf{IS}, E2E is $1.00\times$.}
\label{tab:pwa_main}
\vspace{-5pt}
\renewcommand{\arraystretch}{1.08}
\resizebox{\linewidth}{!}{%
\begin{tabular}{
l
@{\hspace{2.0mm}} c >{\columncolor{gray!10}}c
@{\hspace{2.0mm}} c >{\columncolor{gray!10}}c
@{\hspace{2.0mm}} c >{\columncolor{gray!10}}c
@{\hspace{2.8mm}} c >{\columncolor{gray!10}}c
@{\hspace{2.0mm}} c >{\columncolor{gray!10}}c
@{\hspace{2.0mm}} c >{\columncolor{gray!10}}c
}
\toprule
\multirow{4}{*}{\textbf{Method}}
& \multicolumn{6}{c}{\textbf{2D Manipulation}}
& \multicolumn{6}{c}{\textbf{In-Hand Reorientation}} \\
\cmidrule(lr){2-7}\cmidrule(lr){8-13}
& \multicolumn{2}{c}{\texttt{MPushT}}
& \multicolumn{2}{c}{\texttt{Hook}}
& \multicolumn{2}{c}{\texttt{Shelf}}
& \multicolumn{2}{c}{\texttt{Block}}
& \multicolumn{2}{c}{\texttt{Cuboid}}
& \multicolumn{2}{c}{\texttt{Pen}} \\
\cmidrule(lr){2-3}\cmidrule(lr){4-5}\cmidrule(lr){6-7}
\cmidrule(lr){8-9}\cmidrule(lr){10-11}\cmidrule(lr){12-13}
& \textbf{SR}$\uparrow$ & \textbf{IS}$\uparrow$
& \textbf{SR}$\uparrow$ & \textbf{IS}$\uparrow$
& \textbf{SR}$\uparrow$ & \textbf{IS}$\uparrow$
& \textbf{SR}$\uparrow$ & \textbf{IS}$\uparrow$
& \textbf{SR}$\uparrow$ & \textbf{IS}$\uparrow$
& \textbf{SR}$\uparrow$ & \textbf{IS}$\uparrow$ \\
\midrule

\multicolumn{13}{l}{\cellcolor{gray!25}\textit{- Solver-Oriented Methods}} \\

CoCo
& $0.75$ & $3.57\times$
& $0.63$ & $\mathbf{3.18\times}$
& $0.75$ & $1.39\times$
& $0.69$ & $\mathbf{2.70\times}$
& $0.81$ & $\mathbf{1.61\times}$
& $0.75$ & $1.28\times$ \\

GBD-MPC
& $0.94$ & $0.12\times$
& $0.94$ & $0.06\times$
& $0.81$ & $0.09\times$
& $\mathbf{0.94}$ & $0.32\times$
& $0.88$ & $0.21\times$
& $\mathbf{0.88}$ & $0.06\times$ \\

\multicolumn{13}{l}{\cellcolor{gray!25}\textit{- Symbol-Oriented Methods}} \\

PDDLStream
& $0.88$ & $0.77\times$
& $0.81$ & $0.36\times$
& $0.69$ & $0.42\times$
& $0.81$ & $0.67\times$
& $0.81$ & $0.38\times$
& $0.75$ & $0.17\times$ \\

PGS
& $0.81$ & $0.08\times$
& $0.69$ & $0.37\times$
& $0.75$ & $0.16\times$
& $0.75$ & $0.15\times$
& $0.69$ & $0.13\times$
& $0.81$ & $0.04\times$ \\

IVNTR
& $0.88$ & $0.15\times$
& $0.81$ & $0.22\times$
& $0.75$ & $0.12\times$
& $\mathbf{0.94}$ & $0.10\times$
& $0.75$ & $0.10\times$
& $0.81$ & $0.03\times$ \\

pix2pred
& $0.88$ & $0.03\times$
& $0.88$ & $0.09\times$
& $\mathbf{0.94}$ & $0.08\times$
& $\mathbf{0.94}$ & $0.05\times$
& $0.88$ & $0.04\times$
& $0.75$ & $0.01\times$ \\

\midrule
\textbf{Ours}
& $\mathbf{1.00}$ & $\mathbf{4.81\times}$
& $\mathbf{1.00}$ & $3.07\times$
& $\mathbf{0.94}$ & $\mathbf{2.77\times}$
& $0.88$ & $1.53\times$
& $\mathbf{0.94}$ & $1.34\times$
& $\mathbf{0.88}$ & $\mathbf{1.94\times}$ \\

\bottomrule
\vspace{-20pt}
\end{tabular}%
}
\end{table}

\textbf{Experiment Setup.}
We evaluate our method extensively in both PWA (three 2D manipulation tasks and three in-hand reorientation tasks with multi-finger grippers) and network settings (three D3IL tasks to study symbolic structure discovery in multimodal decision landscapes \citep{jia2024towards}, and two real-world tasks). Our task suites are shown in \cref{fig:suites}. Each entry in our tables is computed over 64 evaluation episodes (8 seeds). The details can be found in Appendix~\ref{app:details}.

\textbf{Baselines.}
For PWA systems, we compare against solver-specific methods CoCo and GBD-MPC, which respectively predict discrete modes before applying convex optimization and refine the original problem through Benders cuts \citep{cauligi2020learning,lin2024accelerating} to accelerate solving. We also compare against existing symbol invention methods: PDDLStream with LLM-invented symbols, PGS, IVNTR, and pix2pred, which construct symbolic abstractions before grounded solving \citep{garrett2020pddlstream,silver2023predicate,li2025bilevel,athalye2026pixels}. CoCo and GBD-MPC accelerate the same MILP without symbols, while the others use symbols not derived from the solver. 
For neural network--based systems, we evaluate multimodal policies BESO, IBC, DDPM, ACT, and BeT under their energy-based variants \citep{reuss2023goal,florence2022implicit,chi2025diffusion,zhao2023learning,shafiullah2022behavior}.
In this paper, the energy function $h_s$ is obtained from a trained policy as its negative log-likelihood up to a constant, or as a surrogate. Each baseline is frozen, with the same solver and budget. The methods for obtaining energy for each baseline are provided in \cref{app:energy}. 
\begin{table}[t]
\centering\small
\caption{\textbf{Solver Sufficiency on Network Systems.}
For \textbf{DC}$_\delta$, we take $\delta=0.001$.}
\label{tab:nn_main}
\vspace{-5pt}
\renewcommand{\arraystretch}{1.08}

\resizebox{\linewidth}{!}{%
\begin{tabular}{
l
@{\hspace{2.0mm}} c >{\columncolor{gray!10}}c
@{\hspace{2.8mm}} c >{\columncolor{gray!10}}c
@{\hspace{2.8mm}} c >{\columncolor{gray!10}}c
@{\hspace{1.8mm}\vrule width 0.7pt\hspace{1.8mm}}
c >{\columncolor{gray!10}}c
@{\hspace{2.8mm}} c >{\columncolor{gray!10}}c
}
\toprule

\multirow{2}{*}{\textbf{Method}}
& \multicolumn{2}{c}{\texttt{Avoiding}}
& \multicolumn{2}{c}{\texttt{Pushing}}
& \multicolumn{2}{c}{\texttt{Aligning}}
& \multicolumn{2}{c}{\texttt{Matching}}
& \multicolumn{2}{c}{\texttt{Sorting}} \\
\cmidrule(lr){2-3}
\cmidrule(lr){4-5}
\cmidrule(lr){6-7}
\cmidrule(lr){8-9}
\cmidrule(lr){10-11}

& \textbf{SR}$\uparrow$ & \textbf{DC}$_\delta\downarrow$
& \textbf{SR}$\uparrow$ & \textbf{DC}$_\delta\downarrow$
& \textbf{SR}$\uparrow$ & \textbf{DC}$_\delta\downarrow$
& \textbf{SR}$\uparrow$ & \textbf{DC}$_\delta\downarrow$
& \textbf{SR}$\uparrow$ & \textbf{DC}$_\delta\downarrow$ \\
\midrule

DDPM
& $\mathbf{1.00}$ & $7.03$
& $0.90$ & $\mathbf{1.00}$
& $\mathbf{0.94}$ & $10.31$
& $\mathbf{0.75}$ & $5.15$
& $\mathbf{0.88}$ & $1.68$ \\

ACT
& $0.91$ & $833.69$
& $0.91$ & $5.93$
& $0.91$ & $9.25$
& $0.69$ & $173.45$
& $0.63$ & $259.67$ \\

BESO
& $0.97$ & $329.03$
& $0.63$ & $198.57$
& $0.75$ & $338.03$
& $\mathbf{0.75}$ & $87.48$
& $0.69$ & $147.85$ \\

IBC
& $0.81$ & $690.34$
& $0.63$ & $312.70$
& $0.63$ & $349.94$
& $0.63$ & $340.78$
& $0.69$ & $562.95$ \\

BeT
& $0.72$ & $495.50$
& $0.30$ & $105.63$
& $0.53$ & $94.66$
& $0.56$ & $936.83$
& $0.75$ & $1325.94$ \\

\midrule

\textbf{DDPM (Symbols)}
& $\mathbf{1.00}$ & $\mathbf{2.00}$
& $\mathbf{0.91}$ & $\mathbf{1.00}$
& $0.84$ & $\mathbf{2.41}$
& $\mathbf{0.75}$ & $\mathbf{1.20}$
& $\mathbf{0.88}$ & $\mathbf{1.00}$ \\

\bottomrule
\end{tabular}%
}

\vspace{-1em}
\end{table}
\vspace{-1em}
\subsection{\textbf{Q1.} Preserving Optimality while Reducing Solving Cost}
\vspace{-0.5em}
\begin{wrapfigure}{r}{0.59\textwidth}
    \vspace{-4mm}
    \centering
    \includegraphics[width=\linewidth]{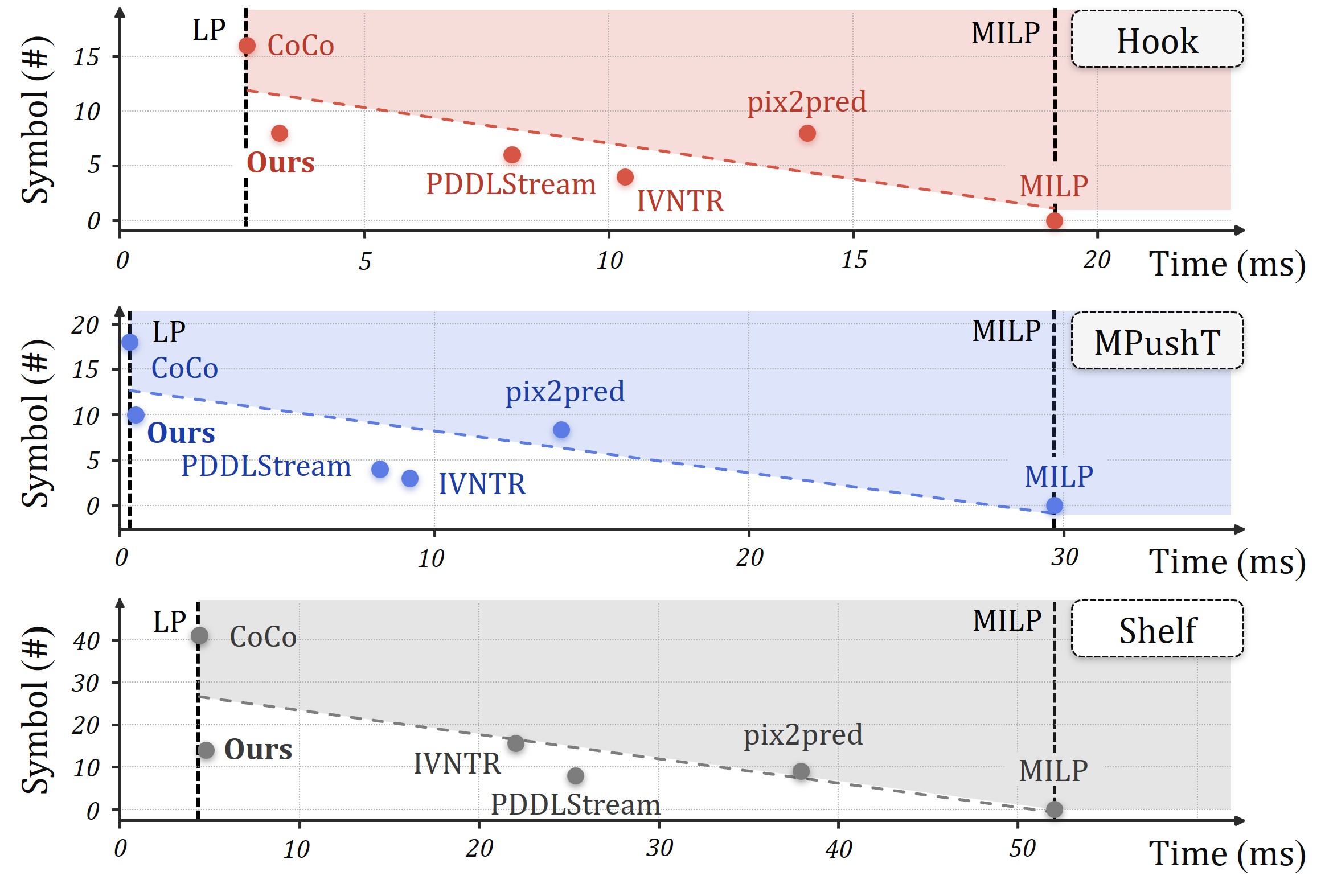}
    \vspace{-20pt}
    \caption{\textbf{Symbol Count vs.\ Conditioned Solving Time.} The dashed line shows the fitted trend across methods.}
    \label{fig:symbol_time}
    \vspace{-7pt}
\end{wrapfigure}
We first evaluate solver sufficiency in the PWA setting using success rate (\textbf{SR}) and per-step inference-time speedup (\textbf{IS}). As shown in \cref{tab:pwa_main}, our symbols largely preserve solution quality while reducing solving time. CoCo obtains efficient conditioned solves by fixing complete discrete assignments to construct an LP, but may exclude the optimum; in contrast, our method exposes only the commitments required for faithful relaxation, leaving the residual problem LP-solvable with a lower-dimensional symbolic search (\cref{fig:symbol_time}). The remaining failures are mainly due to early commitments that invalidate successive transitions: the controller is a receding-horizon controller, and an excessively short horizon may lead to unsuccessful state transitions. It is a general issue of receding-horizon control; a dwell constraint, which holds a symbol for the required number of steps, removes these failures on \texttt{Block}.

For neural settings, we construct $\tau$-descent-faithful symbols on top of a trained DDPM policy without modifying its parameters. We report \textbf{SR} and descents-to-coverage (\textbf{DC}$_\delta$), the number of descents required to reach the $\tau$-neighborhood of the global optimum with probability at least $1-\delta$. We use the reference optimum found by a 200-start multi-start as the global optimum. \textbf{DC}$_\delta$ serves as an indicator of the optimality reachability. As shown in \cref{tab:nn_main}, symbol-guided solving maintains high task success while substantially reducing \textbf{DC}$_{10^{-3}}$, indicating that near-optimal solutions can be recovered with only a few conditioned descents. 
The exception is \texttt{Aligning}, where symbols lower \textbf{SR} from $0.94$ to $0.84$. We attribute the drop to two factors, one for each link of \cref{rem:scope}. On the solver side, a descent that enters a basin shallower than $\tau$ needs additional steps to leave it; raising the step limit per descent brings \textbf{SR} to $0.88$. On the sampling side, a few regions violate $\tau$-descent faithfulness but are not detected at the original sampling density; re-running CPR with $1.2\times$ the samples then removes these violations, and \textbf{SR} returns to $0.94$ with \textbf{DC}$_\delta=2.36$.

We further check optimality directly. In the PWA setting, we fix the symbols to those of the MILP optimum $z^\star$, $\psi^\star=Wz^\star$, and solve the conditioned LP: it attains the MILP optimal value on $100\%$ of the tested states of the six tasks, as the relative affine-TU certificate guarantees. In the neural setting, we run one confined descent in each region and take the best endpoint: it is within $\tau$ of the reference optimum on $95\%$ of the tested states. Symbols thus preserve the optimum exactly in the PWA setting and, up to $\tau$, on most tested states in the neural setting.

\cref{tab:symbol_semantics} further illustrates that our CPR-invented symbols encode
action-relevant commitments, whereas baseline symbols mainly describe state or geometry. Overall, our algorithms and results provide a solver-centric
interpretation of usefulness.

\vspace{-0.75em}
\subsection{\textbf{Q2.} Explaining the Utility of Human-Defined Symbols}
\vspace{-0.5em}


\begin{figure*}[tp]
    \vspace{0em}
    \centering\small
    \includegraphics[width=\linewidth]{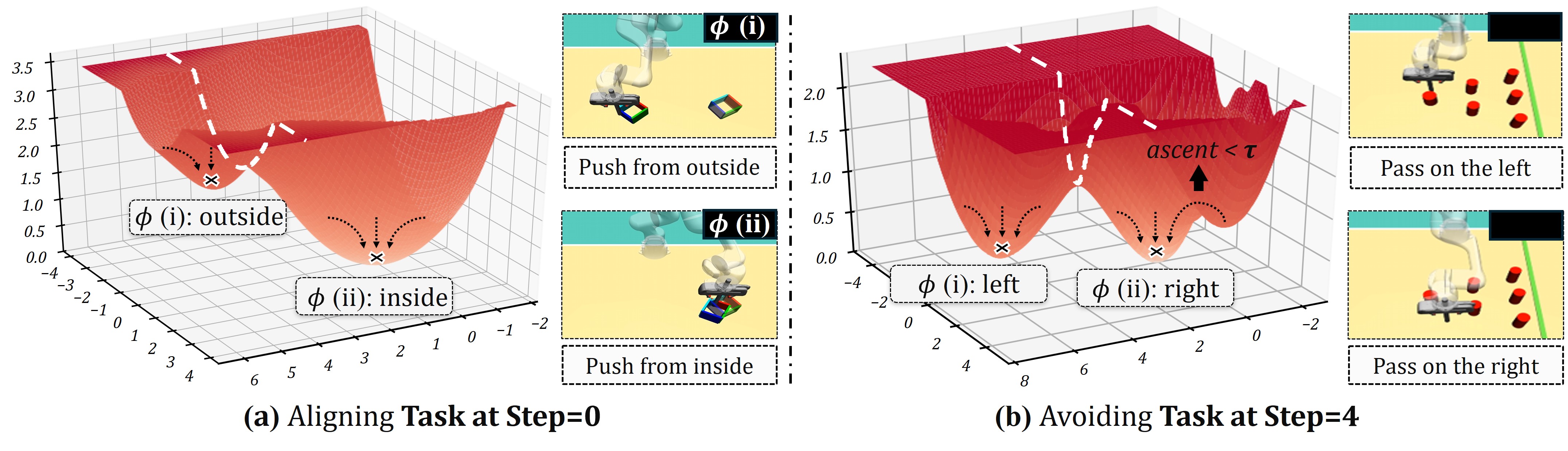}
    \vspace{-20pt}
    \caption{\textbf{Behavior Steering with Invented Symbols.} Different symbol values correspond to distinct basin-level behaviors, allowing steering the controller toward preferences.}
    \label{fig:steering}
    \vspace{-1.5em}
\end{figure*}

To investigate why certain human-defined symbols are computationally useful, we study \texttt{MPushT}, a PushT variant in which four pushers simultaneously move a T-shaped block to a target pose. We manually construct nine MILP formulations containing 127 factors in total (Appendix~\ref{app:mpusht_variation}). We compute a CPR symbolization for each formulation and evaluate each human-defined binary factor $z_i$ along two dimensions. Task necessity (\textbf{TN}) measures the decrease in task success when the factor is disabled. Solver necessity (\textbf{SN}) measures how often the factor is covered by the CPR-invented symbols, defined as the fraction of symbol assignments $\psi' \in \Psi$ for which conditioning on $Wz=\psi'$ uniquely determines $z_i$. 
Since CPR identifies the underlying combinatorial structure that must be preserved under relaxation, a high \textbf{SN} indicates that a factor encodes a discrete distinction that cannot simply be relaxed away by the MILP solver. By contrast, a discrete variable $z_i$ that cannot be relaxed before the CPR symbols are fixed but can be relaxed afterwards has low \textbf{SN}, indicating that it contributes no intrinsic combinatorial hardness. As shown in \cref{fig:human_symbol}, both \textbf{TN} and the fraction of factors with clear human semantics increase with \textbf{SN}, suggesting factors that cannot be simply relaxed are more important for task success. 
For example, $z_2$, a push at the bottom of the stem, has higher
\textbf{SN} and \textbf{TN} than $z_1$, a push at a corner of the top bar. The decision of $z_2$ matters in more task instances, as the contact force of $z_1$ can be replaced by neighboring pushes.

\begin{table}[t]
\centering\small
\caption{\textbf{Semantics of CPR-Invented Symbols.} We only show two invented symbols per setting.}
\label{tab:symbol_semantics}
\vspace{-5pt}

\renewcommand{\arraystretch}{1.08}
\resizebox{\linewidth}{!}{%
\begin{tabular}{llll}
\toprule
\cellcolor{gray!15}\textbf{Task} & \textbf{Ours} & \textbf{PGS} & \textbf{pix2pred} \\
\midrule

\cellcolor{gray!15}\texttt{MPushT}
&
\makecell[l]{
$\displaystyle \sum_{\tau_c<0} z_c-\sum_{\tau_c>0} z_c$
\icmt{is clockwise push}
\\[3pt]
$z_{\mathrm{top}}-z_{\mathrm{bottom}}$
\icmt{is push above}
}
&
\makecell[l]{
$\xi_\theta\leq -0.9$
\icmt{T orientation $\leq$ threshold}
\\[3pt]
$\xi_{\mathrm{dist}}\leq 1.0$
\icmt{near the goal}
}
&
\makecell[l]{
\texttt{stem\_pointing\_down\_right}
\\[3pt]
\texttt{rotation\_error\_small}
}
\\
\midrule

\cellcolor{gray!15}\texttt{Block}
&
\makecell[l]{
$z_{\rm s\rightarrow ff}$
\icmt{index finger contacts object}
\\[3pt]
$z_{\rm th\rightarrow palm}$
\icmt{thumb contacts palm}
}
&
\makecell[l]{
$\xi_{\rm ff\rightarrow thu}>0.5$
\icmt{index finger near thumb}
\\[3pt]
$\xi_{\rm 11}\leq 0.6$
\icmt{control variable $\leq$ threshold}
}
&
\makecell[l]{
\texttt{orientation\_over\_bound}
\\[3pt]
\texttt{orientation\_southeast}
}
\\

\bottomrule
\vspace{-10pt}
\end{tabular}%
}
\end{table}
\begin{table*}[t]
\centering

\begin{minipage}[t]{0.485\textwidth}
\vspace{0pt}
\centering

\captionof{table}{\textbf{Behavior Steering with Symbols.}
See \cref{fig:steering} for preference definitions.}
\label{tab3}
\vspace{-5pt}

\renewcommand{\arraystretch}{1.08}
\setlength{\tabcolsep}{1.1mm}

\resizebox{\linewidth}{!}{%
\begin{tabular}{
l
@{\hspace{1.5mm}} cc cc
@{\hspace{2.0mm}} cc cc
}
\toprule
\multirow{4}{*}{\textbf{Method}}
& \multicolumn{4}{c}{\texttt{Avoiding}}
& \multicolumn{4}{c}{\texttt{Aligning}} \\
\cmidrule(lr){2-5}\cmidrule(lr){6-9}

& \multicolumn{2}{c}{$\phi$ (i)}
& \multicolumn{2}{c}{$\phi$ (ii)}
& \multicolumn{2}{c}{$\phi$ (i)}
& \multicolumn{2}{c}{$\phi$ (ii)} \\
\cmidrule(lr){2-3}\cmidrule(lr){4-5}
\cmidrule(lr){6-7}\cmidrule(lr){8-9}

& \textbf{SR} & \textbf{PR}
& \textbf{SR} & \textbf{PR}
& \textbf{SR} & \textbf{PR}
& \textbf{SR} & \textbf{PR} \\
\midrule

\textbf{Ours}
& $1.00$ & $0.94$
& $1.00$ & $1.00$
& $0.75$ & $0.75$
& $0.83$ & $0.81$ \\

\bottomrule
\end{tabular}%
}

\end{minipage}
\hfill
\begin{minipage}[t]{0.46\textwidth}
\vspace{0pt}
\centering\small
\captionof{table}{\textbf{Invented Symbol Complexity Scale with Planning Horizon.}}
\label{tab:horizon_scaling}
\vspace{-5pt}

\renewcommand{\arraystretch}{1.08}
\setlength{\tabcolsep}{1.1mm}

\resizebox{\linewidth}{!}{%
\begin{tabular}{lcccccccc}
\toprule
& \multicolumn{2}{c}{$H=4$}
& \multicolumn{2}{c}{$H=8$}
& \multicolumn{2}{c}{$H=12$}
& \multicolumn{2}{c}{$H=16$} \\
\cmidrule(lr){2-3}
\cmidrule(lr){4-5}
\cmidrule(lr){6-7}
\cmidrule(lr){8-9}

\textbf{Task}
& $k$ & \textbf{SR}
& $k$ & \textbf{SR}
& $k$ & \textbf{SR}
& $k$ & \textbf{SR} \\
\midrule

\texttt{Avoiding}
& 0 & 1.00
& 1 & 1.00
& 1 & 0.91
& 2 & 0.94 \\

\texttt{Aligning}
& 4 & 0.84
& 5 & 0.75
& 6 & 0.69
& 8 & 0.63 \\

\bottomrule
\end{tabular}%
}

\end{minipage}
\end{table*}

\vspace{-0.5em}
\subsection{\textbf{Q3.} Retaining Symbolic Interpretability and Steerability}
\vspace{-0.5em}
\begin{wrapfigure}{r}{0.45\textwidth}
    \vspace{-4mm}
    \centering
    \includegraphics[width=\linewidth]{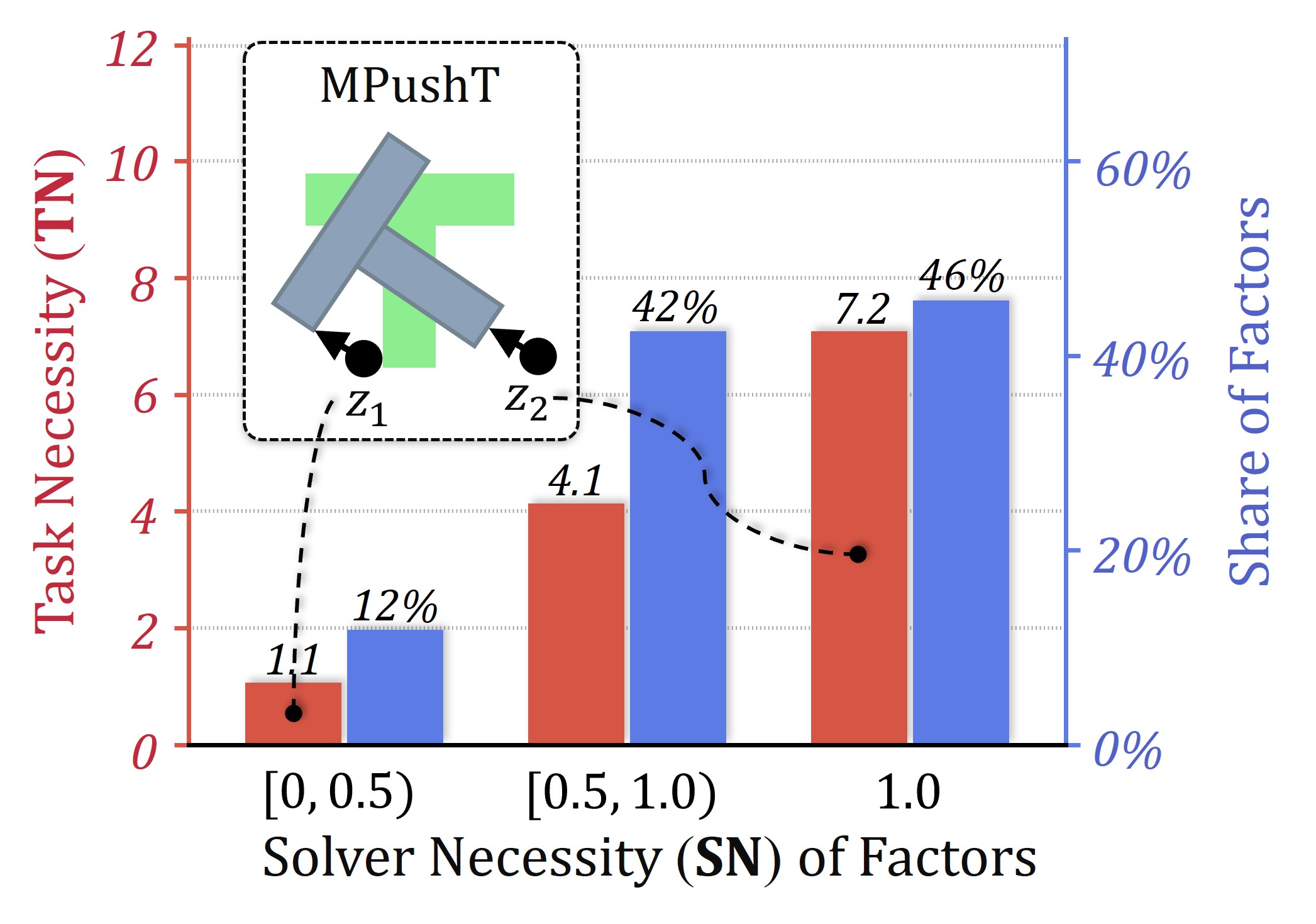}
    \vspace{-20pt}
    \caption{\textbf{Human Concept Analysis.}}
    \label{fig:human_symbol}
    \vspace{-10pt}
\end{wrapfigure}
Like prior symbolic representations \citep{pan2026inference,mao2019neuro}, our invented symbols also support behavior steering through motion-level preferences. As shown in \cref{fig:steering}, different manipulation strategies will emerge in policies trained for \texttt{Aligning} (pushing from inside versus outside), and \texttt{Avoiding} (passing an obstacle on the left versus right). These alternative strategies correspond to distinct basins of the underlying energy landscape $h_s$, which CPR will discover as separate symbol regions. Because the symbol map is defined over $x$ (\cref{eq:neural_symbol_map}), fixing a symbol value $\psi$ restricts the control command $x$ to the corresponding region $\Phi^{-1}(\psi)$. To make these regions interpretable, we use a vision-language model to label rollout videos from different symbol regions and assign each region a behavioral label. For example, the left basin in \cref{fig:steering} (a) is labeled ``pushing from outside.'' At inference time, we can manually set the symbol values associated with the desired preference. As shown in \cref{tab3}, this yields both high \textbf{SR} and high preference success rates (\textbf{PR}: the fraction of successful trials that also satisfy the preference).

\begin{wraptable}{r}{0.48\textwidth}
\vspace{-12pt}
\centering\small
\caption{\textbf{Comparison with Clustering.} Same energy, solver, and budget.}
\label{tab:clustering}
\vspace{-5pt}

\renewcommand{\arraystretch}{1.08}
\setlength{\tabcolsep}{2.0mm}

\resizebox{\linewidth}{!}{%
\begin{tabular}{
l
@{\hspace{4.0mm}} c >{\columncolor{gray!10}}c
@{\hspace{4.0mm}} c >{\columncolor{gray!10}}c
}
\toprule
\multirow{2}{*}{\textbf{Method}}
& \multicolumn{2}{c}{\texttt{Avoiding}}
& \multicolumn{2}{c}{\texttt{Pushing}} \\
\cmidrule(lr){2-3}\cmidrule(lr){4-5}

& \textbf{SR}$\uparrow$ & \textbf{DC}$_\delta\downarrow$
& \textbf{SR}$\uparrow$ & \textbf{DC}$_\delta\downarrow$ \\
\midrule

DDPM
& $\mathbf{1.00}$ & $7.03$
& $0.90$ & $\mathbf{1.00}$ \\

Action Clustering
& $\mathbf{1.00}$ & $16.25$
& $0.88$ & $4.00$ \\

Output Clustering
& $\mathbf{1.00}$ & $21.25$
& $0.84$ & $3.62$ \\

\midrule

\textbf{DDPM (Symbols)}
& $\mathbf{1.00}$ & $\mathbf{2.00}$
& $\mathbf{0.91}$ & $\mathbf{1.00}$ \\

\bottomrule
\end{tabular}%
}

\vspace{-5pt}
\end{wraptable}

\subsection{Ablation Studies}
\paragraph{Comparison with clustering.}
Clustering-based methods are a particular class of symbol invention: they extract symbols from a given dataset \citep{pan2026emergent,konidaris2018skills}. Our method instead queries only the system itself, bypassing the need for such a dataset. For a controlled comparison, we apply HDBSCAN to the same network, clustering action samples (action modes) or endpoints of descents on $h_s$ (output modes); each cluster defines one region. Clustering needs more descents than CPR symbols and unpartitioned DDPM, without improving \textbf{SR} (\cref{tab:clustering}): clearly separated semantic modes do not imply solver sufficiency.

\paragraph{Horizon-induced symbol complexity.}
As the planning horizon grows, the decision landscape becomes more complex through accumulated long-horizon state-action dependencies, and the required symbol dimension $k$ generally increases (\cref{tab:horizon_scaling}). The growth may plateau once dominant modes are exposed. Yet the resulting increase in basin complexity also makes optimization harder, leading to a moderate decline in \textbf{SR}. Tasks with clearer decision modes scale more favorably: for example, \texttt{Avoiding}, which mainly requires choosing a passing direction around obstacles, needs far fewer symbols than \texttt{Aligning}.

\paragraph{Invention cost.}
Our symbol invention is a one-time offline step: CPR runs once per task in the PWA setting and once per horizon in the neural setting. Its cost is low: it takes seconds on PWA tasks and minutes on neural policies (\cref{tab:tau_ablation}), and, unlike methods that learn symbols from data, it requires neither a dataset nor training. Details are in Appendix~\ref{app:invention_cost}.

\vspace{-1em}
\section{Conclusion}
\vspace{-0.5em}

\xhdr{Limitations.}
So far, our formulation uses linear symbol models $u^\top x$, which may be restrictive in some cases. Yet, if task-relevant modes are sufficiently abstracted and linearly separable in a learned feature space, $u^\top x$ can still partition the landscape effectively; otherwise, nonlinear maps $u(x)$ may be preferable. We include more discussions and ablation studies in Appendix~\ref{app:further}. Moreover, in the neural setting, the certificate holds only relative to the samples. Finally, the ability to consistently navigate over barriers with height less than $\tau$ during a confined descent is limited by the capabilities of the given optimizer.

We presented an inference-algorithm--centric view of symbol invention: useful symbols expose discrete commitments that simplify planning. We formalized this principle for piecewise-affine systems and extended it to neural energy landscapes. Across simulated and real-robot tasks, the invented symbols reduce solving cost while preserving solution quality and provide an interpretable interface for behavior steering. These results connect symbolic abstraction to computational utility, offering a basis for both inventing new symbols and understanding the usefulness of human-defined concepts.

\newpage

\bibliography{iclr2027_conference}
\bibliographystyle{iclr2027_conference}

\newpage
\appendix

\startcontents[appendix]

\section*{Appendix Contents}
{
\setcounter{tocdepth}{2}
\printcontents[appendix]{}{1}{}
}

\section{Related Works}\label{app:related_works}
\subsection{Search-then-Optimize: From MILP to Energy-Based Models}
\label{app:rw_partition}
Many planning methods share a two-stage shape: commit to a discrete quantity, then solve a conditioned problem that is tractable once it is fixed. Optimization formalizes this as the partitioned optimization framework, where minimizing over partition indices recovers solutions of the original problem \citep{audet2026partitioned}. For piecewise-affine models written as MILPs, the index is the mode assignment and each conditioned problem is convex; methods differ in how that assignment is obtained, by warm-starting branch-and-bound \citep{marcucci2020warm}, by predicting it so that a convex program remains online \citep{cauligi2021coco}, or by transferring Benders cuts across instances \citep{lin2024accelerating}. In each, the discrete structure is declared by the formulation and the effort goes into searching it faster; which partition to use, and whether the conditioned problem is tractable under it, remain open \citep{audet2026partitioned,pan2026seeing,pan2026wonder}.

When the objective is a learned network, the shape persists but no partition is declared. Energy-based and diffusion policies represent alternative valid behaviors as distinct basins of the learned objective \citep{florence2022implicit,chi2025diffusion}, and which basin a run settles in is decided by its initialization. This is how multimodality is represented, not a defect, but it makes the basin sampled rather than decided. Never represented as a variable, it cannot be committed to or held fixed across successive solves. Adjacent action chunks, drawn from independent noise, may therefore select incompatible modes \citep{black2026real}.

Existing remedies act on the generated actions rather than on the basin: averaging overlapping predictions \citep{zhao2023learning}, selecting among candidates for consistency with previous decisions \citep{liu2025bidirectional}, or correcting the denoising process by inpainting \citep{black2026real}; each enforces agreement between consecutive outputs without making the basin explicit. We instead supply an explicit index together with what fixing it buys, at two strengths: for PWA systems the LP relaxation of each conditioned problem is exact (\cref{thm:rat_complexity}), whereas for learned objectives, under the solver assumption of \cref{rem:scope}, one confined descent per region reaches within $\tau$ of the global optimum regardless of its start (\cref{prop:tau_guarantee}), certified from sampled descents and therefore one-sided over the continuous domain.

\subsection{Symbol Invention in Neuro-Symbolic Planning}
\label{app:rw_symbol}
Symbol invention involves two questions: where candidates come from, and which to keep; existing methods differ mainly in the first. Symbols have been grounded in the preconditions and effects of given skills \citep{konidaris2018skills}, taken as the discrete code of an autoencoder \citep{asai2018classical}, tied to a language vocabulary \citep{mao2019neuro}, enumerated from a grammar over given object-centric features \citep{silver2023predicate,chitnis2022learning}, learned as neural classifiers supervised by operator effects \citep{li2025bilevel}, or proposed by a pretrained vision-language model \citep{athalye2025predicate,liang2025visualpredicator}; with the set fixed, planners interleave symbolic search with continuous samplers \citep{garrett2020pddlstream}. In each case the admissible symbols are delimited before learning begins, by a grammar, a feature set, a code size, or a pretrained vocabulary.

The second question is answered by measurement. Candidates are scored by the likelihood they assign to observed transitions \citep{pasula2007learning}, or by planning performance on a distribution of training tasks under a particular planner \citep{silver2023predicate}. Such scores are informative but extrinsic: they report how a symbol set behaved, not what fixing its value does to the remaining problem. They also do not single out a set. The same task admits many symbol sets whose measured performance is close, and a success rate then offers no basis for choosing among them.

We take both the source and the criterion from the solver itself. Candidates are read off the structure that makes solving hard---the discrete--continuous interface for PWA systems, and the basins revealed by sampled descents for learned objectives---so no vocabulary is delimited in advance. The criterion is a property of the conditioned problem rather than a measurement: an exact LP relaxation for PWA systems, and, at the weaker strength that sampling permits, $\tau$-descent faithfulness for learned objectives (\cref{prop:tau_guarantee}). The count $k$ is then itself a measure of difficulty---the commitments that cannot be relaxed away---with the computable lower bound of \cref{prop:rat_bounds}.

\subsection{Planning Efficiency and Structural Coupling}
\label{app:rw_coupling}
This paper is concerned with planning efficiency: what a solve actually costs on a given problem. Plan existence being PSPACE-hard in general \citep{bylander1994computational}, efficiency is studied through finer structural measures, separating along two axes---how many steps must be searched, and how much must be considered jointly at each step. In heuristic search these are the expansions to the goal and the branching kept at each expansion; \citep{mao2023planning} make the same separation explicit for relational policies, classifying problems by how a policy circuit's width and depth grow with object count and horizon. Work on the first axis replaces primitive steps with temporal abstractions \citep{sutton1999between}; work on the second prunes states whose novelty exceeds a bound \citep{geffner2012width}.

Both axes trace to one property: coupling among the decisions in a plan---the extent to which they cannot be made independently. Width notions make it operational: how many atoms must hold together before a subgoal can be treated in isolation \citep{chen2007act,geffner2012width}, and the regression width of \citet{mao2023planning}, $|c|+\max_i|c_i|$, counts the conditions that must be maintained while a subgoal is achieved. When subgoals serialize, search factors and both axes shrink; otherwise they must be decided jointly. The induced width of a constraint graph measures the same thing \citep{dechter2007constraint}.
 
All of these are defined on a declared combinatorial structure: a predicate vocabulary with preconditions and effects, or a constraint graph. 
We measure coupling instead by whether the solver's relaxation is faithful using affine TU tools \citep{hoffman2010integral,bader2018mixed} for PWA systems,
sublevel persistence \citep{edelsbrunner2002topological} for learned objectives. The two are not interchangeable: total unimodularity is not a property of the constraint graph, since matrices with the same sparsity pattern can differ in it by sign alone. RATdim counts the commitments required before the LP relaxation is exact (\cref{prop:rat_bounds}), and fixing them decouples the discrete and continuous sides; $\tau$-descent faithfulness asks the same question where no transition structure is declared, with a certificate established on sampled descents rather than exactly

\section{Theoretical Foundations}\label{app:proof}
\subsection{Structural Assumptions on the Discrete Interface}
\label{app:assumptions}
\begin{assumption}
\label{ass:discrete}
The discrete matrices in \cref{eq:pwa_joint_problem} satisfy: \emph{(1)} $E^d$ is TU; \emph{(2)} $B^d\in\{0,\pm1\}^{m_B\times n_d}$.
\end{assumption}

Physically, both are mild. Condition~(1) holds when each mode variable is selected by at most one equality row, as in per-contact one-hot mode logic. Condition~(2) says that a contact either enters a bridge constraint or does not, and if it does, only with a direction; the magnitude is carried by $B^c$ and $b$. To enforce $B^d$ to be a ternary matrix, we can divide each bridge row by the magnitude of its discrete coefficients, which is available for big-$M$ gates and signed incidence.

\subsection{Numerical Complexity of Symbolic Conditioning}\label{app:rat_complexity}
This section quantifies the structural simplification induced by solver-sufficient symbols through the numerical complexity. We formally give the following proposition.
\begin{proposition}[RAT Integer-Dimension Complexity]
\label{thm:rat_complexity}
(Proof in \cref{app:rat_complexity}) Under \cref{ass:discrete}, let $W$ be a $k$-row relative affine-TU certificate, $\mathcal{P}_{\mathrm{PWA}}^{\psi}$ admits the formulation
\begin{equation}
    \mathcal{P}_{\rm PWA}^\psi: \ \min\nolimits_{z,\xi}\ c_d^\top z+c_c^\top \xi \qquad \mathrm{s.t.}\quad (z,\xi)\in \mathcal{P}_{\rm PWA}, \ Wz=\psi. \label{eq:rat_reduced_milp}
\end{equation}
For every fixed $\psi$, \cref{eq:rat_reduced_milp} is an LP. Let $m$ and $L$ be the number of constraints and input encoding length. Searching over $\psi$ and then optimizing can be solved in
\[
2^{\mathcal O(k\log n_d)}\operatorname{poly}(n_d,n_c,m,L).
\] 
The same reduction can also be viewed through Bellman lens, as shown in \cref{app:mdp}.
\end{proposition}

\cref{thm:rat_complexity} states that a $k$-row certificate moves the exponential dependence from the $n_d$ primitive modes to the $k$ symbols. We first recall Lenstra's theorem, which converts a bound on the number of integer variables into a running time. 

\begin{lemma}[Lenstra Theorem]\label{lem:lenstra}
A mixed-integer linear program with $n$ integer variables, arbitrary continuous
variables, $m$ constraints and rational data of encoding length $L$ is solvable in
\[
    2^{\mathcal O(n^{3})}\operatorname{poly}(n,m,L).
\]
\end{lemma}

\begin{proof}
After fixing the symbol $\psi$, the feasible set factorizes as in follows, and the objective separates accordingly as well. 
\begin{subequations}
\label{eq:conditioned_sets}
\begin{align}
{\mathcal F}^{d}(\psi)
&=
\left\{
z\in[0,1]^{n_d}:
E^dz=e^d,\;
Wz=\psi,\;
\widetilde P^dz\leq p^d-U_P\psi-V_Pe^d
\right\},
\label{eq:conditioned_discrete}
\\
\mathcal F^{c}(\psi)
&=
\left\{
\xi\in\mathbb{R}^{n_c}:
E^c\xi=e^c,\;
P^c\xi\leq p^c,\;
B^c\xi\leq b-U_B\psi-V_Be^d
\right\}.
\label{eq:conditioned_continuous}
\end{align}
\end{subequations}
By \cref{theo:hoffman}, $\mathcal F^{d}(\psi)$ is an integral polytope: relaxing $z$ to $[0,1]^{n_d}$ loses nothing, and $\mathcal P^{\psi}_{\rm PWA}$ is solved by a single LP. And such transform preserves the optimality of the optimization problem, since the vertices of the solution automatically fall on integer points. Therefore, the lifted program has only $k$ exact integer variables, so \cref{lem:lenstra} replaces the direct bound $2^{\mathcal O(n_d^{3})}\operatorname{poly}(\cdot)$ by $2^{\mathcal O(k^{3})}\operatorname{poly}(n_d,n_c,m,L)$.

However, we can further accelerate the solution process. In practice, we does not solve the lifted program and enumerate over remaining $k$ discrete variables; we instead searches the $\psi$-symbol space directly. Each row of $W$ lies in $\{0,\pm1\}^{n_d}$ and $z$ is binary, so $w_j^{\top}z\in[-n_d,n_d]$ and $\lvert\Psi\rvert\le(2n_d+1)^{k}$. Solving one LP under a fixed symbol commitment $\psi$ is exact, giving $2^{\mathcal O(k\log n_d)}\operatorname{poly}(n_d,n_c,m,L)$.
\end{proof}

\subsection{Bounds on Relative Affine-TU Dimension}\label{app:rat_bounds}
\cref{prop:rat_bounds} brackets the relative affine-TU dimension between a rank floor and the primitive dimension $n_d$, with the following proof.

\begin{proof}
Let $W$ be a certificate and $U_B,V_B$ the rational multipliers in condition~(2). Then $B^d=[\,U_B\ \ V_B\,]\,[W;E^d]$, so every row of $B^d$ lies in $\operatorname{rowspan}[W;E^d]$, whence $\operatorname{rowspan}[B^d;E^d]\;\subseteq\;\operatorname{rowspan}[W;E^d]$. Taking dimensions and using sub-additivity of rank under vertical concatenation,
\[
    \operatorname{rank}[B^d;E^d]\;\le\;\operatorname{rank}[W;E^d]
    \;\le\;\operatorname{rank}(W)+\operatorname{rank}(E^d)
    \;\le\;k+\operatorname{rank}(E^d),
\]
the last step because $W$ has $k$ rows. Rearranging gives the lower bound. For the right-hand bound, Since we can choose $W=\mathbf{I}_{n_d\times n_d}$ in the worst case, so it also holds.
\end{proof}

\subsection{Bellman Factorization for Relative Affine-TU}\label{app:mdp}
We first build a proposition before introducing the Bellman factorization of relative affine-TU. In our settings, the matrices $E^d,P^d,B^d$ describe the full planning horizon and encode cross-step couplings. In many robotic formulations, however, the same local constraint structure repeats at each time step. Based on the circumstance, we propose the following proposition, where the same symbolic structure can be reused across time.
\begin{proposition}[Temporal Separation of Relative Affine-TU]
\label{prop:rat_temporal_separation}
Suppose
\[
E^d=I_T\otimes \bar E^d,\qquad
P^d=I_T\otimes \bar P^d,\qquad
B^d=I_T\otimes \bar B^d.
\]
If $\bar W$ is a $\bar k$-row relative affine-TU certificate for $(\bar E^d,\bar P^d,\bar B^d)$, then $W=I_T\otimes \bar W$ is a relative affine-TU certificate for the full-horizon problem, with $k=T\cdot \bar k$.
\end{proposition}

\cref{thm:rat_complexity} treats $\mathcal P_{\rm PWA}$ as one monolithic program. When the interface repeats in time, \cref{prop:rat_temporal_separation} lets the same certificate be reused at every step, and the reduction can be stated inside the dynamic program. We show that the symbol layer factorizes each Bellman backup exactly: it prunes mode branching within a backup without altering the value function.

First, we let $\bar W$ be a $\bar k$-row certificate for $(\bar E^{d},\bar P^{d},\bar B^{d})$ as in \cref{prop:rat_temporal_separation}, write $\xi=(s,u)$, and let the one-step constraint in MDP set be
\[
\mathcal C(s)=\bigl\{\bar E^{d}z=e^{d},\ \bar P^{d}z\le p^{d},\
\bar E^{c}\xi=e^{c},\ \bar P^{c}\xi\le p^{c},\ \bar B^{d}z+\bar B^{c}\xi\le b\bigr\}.
\]
Assume the stage cost is separable and the successor state $s'(\xi)$ does not depend on $z$. The Bellman operator is $(\mathcal TV)(s)=\min_{z\in\{0,1\}^{\bar n_d},\,\xi}\{c_d^{\top}z+c_c^{\top}\xi+V(s'(\xi)):\mathcal C(s)\}$.

\begin{proposition}[Relative affine-TU Bellman Backup]\label{prop:rat_bellman}
Let $\Psi=\bar W\{0,1\}^{\bar n_d}$, for every $V$ and $s$,
\[
(\mathcal TV)(s)=\min_{\psi\in\Psi}\bigl[g(\psi)+J_V(s,\psi)\bigr],\
g(\psi)=\min_{z\in\mathcal F^{d}(\psi)}c_d^{\top}z,\
J_V(s,\psi)=\min_{\xi\in\mathcal F^{c}(\psi)}\bigl\{c_c^{\top}\xi+V(s'(\xi))\bigr\},
\]
with $\mathcal F^{d},\mathcal F^{c}$ as in \cref{eq:conditioned_sets}. The identity is exact: the value function and its fixed point are unchanged.
\end{proposition}

\begin{proof}
The map $W$ partitions $z\in\{0,1\}^{\bar n_d}$ into the subproblems indexed by $\psi$, so minimizing over $z$ equals minimizing over $\psi$ and then over its subproblem. This step is a rewriting of the feasible set and uses no property of $\bar W$.

Fixing $\psi$ on $\{\bar E^{d}z=e^{d},\ \bar Wz=\psi\}$, conditions~(1) and~(2) of \cref{def:rat_certificate} give $\bar P^{d}z=\widetilde P^{d}z+U_P\psi+V_Pe^{d}$ and $\bar B^{d}z=U_B\psi+V_Be^{d}$. The second removes $z$ from the coupling row, leaving $\bar B^{c}\xi\le b-U_B\psi-V_Be^{d}$; hence the inner feasible set is the product $\mathcal F^{d}(\psi)\times\mathcal F^{c}(\psi)$. Since the stage cost is separable and $s'$ depends only on $\xi$, the inner objective splits, and the minimum of a separable objective over a product set is the sum of the two minima, giving $g(\psi)+J_V(s,\psi)$.

$\psi\in\Psi$ is integral because $\bar W\in\{0,\pm1\}^{\bar k\times\bar n_d}$ and $z$ is binary; the right-hand side of $\mathcal F^{d}(\psi)$ is therefore integral. Condition~(3) makes the defining system TU, so by \cref{theo:hoffman} the polytope is integral and $g(\psi)$ is attained at a binary $z$ even after relaxing $z$ to $[0,1]^{\bar n_d}$. Thus no step of the rewriting discards a feasible point or changes an optimal value, and $\mathcal T$ is preserved; applying it along $V_t=\mathcal TV_{t+1}$ leaves every $V_t$ unchanged.
\end{proof}

\begin{remark}
The backup no longer enumerates the $2^{\bar n_d}$ primitive mode assignments: its combinatorial
part is a search over $\Psi$, with $\lvert\Psi\rvert\le(2\bar n_d+1)^{\bar k}$, followed by one
LP for $g$ and one continuous subproblem for $J_V$. The factorization acts \emph{inside} each
backup and does not shorten the horizon; it is action-level pruning, not temporal abstraction.
\end{remark}

\subsection{Guarantee of a $\tau$-Faithfulness}\label{app:tau_guarantee}
\cref{thm:rat_complexity} is what makes $\psi=Wz$ solver-sufficient in the PWA setting. The statement below is its counterpart for \cref{def:tau_univalence}, and discharges the two requirements of the solver-sufficient symbols on the network side: (i) Searching over symbol assignments preserves optimality up to the solver resolution $\tau$; and (ii) Under the solver assumption of \cref{rem:scope}, each assignment reduces the conditioned problem to a single confined descent, with no dependence on where that descent starts.

\begin{proposition}[$\tau$-Descent Faithfulness Guarantee]\label{prop:tau_guarantee}
Let the symbol map $\Psi :=W$ of \cref{eq:neural_symbol_map} induce the conditioned subproblems $R_\psi:=\Psi^{-1}(\psi)$, $\psi\in \Psi(\mathcal X)$, and suppose every $R_\psi$ is $\tau$-descent faithful for $h_s$. Then for any initializations $x_\psi\in R_\psi$,
\[
    \min\nolimits_{\psi\in W(\mathcal X)} v^{\tau}_{s,R_\psi}(x_\psi)
    \;\le\;
    \min\nolimits_{x\in\mathcal X}h_s(x)+\tau ,
\]
at a cost of $\lvert W(\mathcal X)\rvert$ confined descents under the solver assumption of \cref{rem:scope}. If moreover $h_s$ is continuous on $\mathcal X$, such a $W$ exists.
\end{proposition}

\begin{proof}
Let $R^{\star}$ be a subproblem containing a global minimizer, so $\min_{R^{\star}}h_s=\min_{\mathcal X}h_s$; by \cref{def:tau_univalence} applied to $R^{\star}$, $v^{\tau}_{s,R^{\star}}(x)\le\min_{R^{\star}}h_s+\tau$ for \emph{every} $x\in R^{\star}$, and the left-hand side, being a minimum over all subproblems, is at most its value at $x_{\star}$. For existence, uniform continuity of $h_s$ on the compact $\mathcal X$ gives $\delta>0$ with $\lvert h_s(x)-h_s(y)\rvert\le\tau$ whenever $\lVert x-y\rVert_{\infty}\le\delta$, so any symbol map makes each subproblem vary by at most $\tau$ and hence $\tau$-descent faithful.
\end{proof}

\section{Details of CPR for PWA Systems}
\label{app:pwa_cpr}
CPR follows CEGAR: construct a symbolization, verify it, and refine failed candidates. It groups identical columns of $[B^d;E^d;P^d]$ since they do not affect TU verification; write $(E,P,B)$ for the reduced matrices. \textit{Firstly}, starting from $\varnothing$, $W_0$ appends each row of $B$ whose integer coefficients cannot be recovered from the currently selected rows and $E$, based on \cref{prop:rat_bounds}. Note the rank lower bound is a compression target, not a guaranteed initialization size.

Given $W_i$, CPR next removes from $P$ the constraints already explained by the current symbols and $E$. Specifically, it constructs $\tilde P_i := P - U_{P,i}W_i-V_{P,i}E$. CPR then verifies the complete matrix $M_i=[\widetilde P_i;E;W_i]$ with the TU oracle. If $M_i$ is TU, $W_i$ is certified and returned; otherwise, CPR first attempts a dimension-preserving repair (\textbf{\textsc{Sparsify}}-step) and, if needed, refines $W_i$ with an additional commitment (\textbf{\textsc{Split}}-step).

\paragraph{1) \textsc{Sparsify}-Step.}
When $M_i$ is non-TU, CPR first asks whether the failure is caused only by an unfavorable basis of the current symbol space (false obstruction), rather than by a missing symbol. CPR therefore searches for alternative $\{0,\pm1\}$ bases of the same row space through bounded integer row operations, preserving $k$ while reducing row density. It additionally performs a one-time, budgeted search at the rank floor over $\operatorname{rowspan}([B;E])$. Each candidate is accepted only if bridge recovery remains valid and the recomputed $M_i$ passes the TU oracle; failure of this bounded search does not certify that a larger $k$ is necessary.

\paragraph{2) \textsc{Split}-Step.}
If the recomputed $M_i$ remains non-TU after \textbf{\textsc{Sparsify}}, CPR treats the failure as a genuine obstruction under the current symbolization and introduces an additional commitment through \textbf{\textsc{Split}}. It first localizes an inclusion-minimal non-TU core by greedily deleting rows and columns while preserving non-TU, then considers implicated columns that still participate in multi-entry rows of $W_i$. For each candidate column $c$, CPR removes its contribution from the existing rows and appends the unit row $e_c^\top$, making $e_c^\top z=z_c$ explicit:
\[
W_i^{c}=\bigl[\,W_i-(W_i)_{:,c}e_c^\top;\;e_c^\top\,\bigr].
\]
Each trial is residualized again to form its verification matrix $M^c_i=[\widetilde P_i;E;W^c_i]$ and is ranked first by TU success, then by row count. The selected split becomes $W_{i+1}$ and re-enters the same verification loop; once the full RAT conditions pass, the certificate is expanded back to the original column groups and returned.

\begin{algorithm}[H]
\footnotesize
\caption{Detailed CPR Algorithm (PWA Systems)}
\label{alg:pwa_cpr_detailed}
\textbf{Init:} PWA matrices $(E^{d},P^{d},B^{d})$; TU oracle \textsc{IsTU}\\
\textbf{Certificate:} Relative affine-TU matrix $W$
\vspace{2pt}\hrule\vspace{4pt}
\begin{algorithmic}[1]

\STATE Group identical columns of $[B^d;E^d;P^d]$ and retain reduced matrices $(E,P,B)$.
\STATE Construct $W_0$ from directions of $B$ not recoverable from $E$;
\icmt{\cref{prop:rat_bounds}}

\REPEAT

\STATE \tikzmarknode{PV1}{
Construct $\widetilde P_i=P-U_{P,i}W_i-V_{P,i}E$.
}

\STATE \tikzmarknode{PV2}{
$M_i\gets[\widetilde P_i;E;W_i]$.
}

\STATE \tikzmarknode{PV3}{
\textbf{if} \textsc{IsTU}$(M_i)$, \textbf{return} $W_i$.
} \\

\icmt{If fails, first test whether the obstruction in $W_i$ is induced by the current basis.}

\STATE \tikzmarknode{PS1}{
Re-base $W_i$ within the same row space using bounded integer row operations.
}

\STATE \tikzmarknode{PS2}{
Optionally search, once and within budget, for a verified
$\{0,\pm1\}$ basis at the rank lower bound.
}

\STATE \tikzmarknode{PS3}{
Recompute $\widetilde P_i$ and $M_i$.
\textbf{If} \textsc{IsTU}$(M_i)$, \textbf{return} $W_i$.
}

\STATE \tikzmarknode{PT1}{
$O_i\gets$ an inclusion-minimal non-TU core of $M_i$.
}

\FOR{each eligible $c\in\operatorname{cols}(O_i)$}

    \STATE \tikzmarknode{PT2}{
    $W_i^c\gets[\,W_i-(W_i)_{:,c}e_c^\top;\;e_c^\top\,]$;
    remove zero and duplicate rows.
    }

    \STATE \tikzmarknode{PT3}{
    Recompute $\widetilde P_i^c$ and
    $M_i^c=[\widetilde P_i^c;E;W_i^c]$.
    }

\ENDFOR

\STATE \tikzmarknode{PT4}{
Select a TU-passing $W_i^c$ if one exists; otherwise select the candidate with the fewest rows.
}

\STATE \tikzmarknode{PT5}{
$W_{i+1}\gets W_i^c$, $i\gets i+1$.
}
\icmt{New symbol, $k\rightarrow k+1$}

\UNTIL{$W_i$ is certified or $k=n_d$}

\STATE \textbf{Return} $W_i$.

\end{algorithmic}

\begin{tikzpicture}[
    remember picture,
    overlay,
    bx/.style={
        draw=#1,
        densely dotted,
        line width=1pt,
        rounded corners=2pt,
        fill=none,
        inner xsep=1pt,
        inner ysep=0pt
    },
    bx/.default=black!55
]

\node[
    bx=cVerify,
    fit={(PV1)(PV2)(PV3)([xshift=\algw]PV1.west)}
] (VerifyBox) {};
\node[
    anchor=south east,
    text=cVerify,
    font=\scriptsize
] at ([yshift=-1.5pt,xshift=3pt]VerifyBox.north east)
{\textbf{\textsc{Verify}}};

\node[
    bx=cSparsify,
    fit={(PS1)(PS2)([xshift=\algw]PS1.west)}
] (SparsifyBox) {};
\node[
    anchor=south east,
    text=cSparsify,
    font=\scriptsize
] at ([yshift=-1.5pt,xshift=3pt]SparsifyBox.north east)
{\textbf{\textsc{Sparsify}}};

\node[
    bx=cSplit,
    fit={(PT1)(PT2)(PT3)(PT4)(PT5)([xshift=\algw]PT1.west)}
] (SplitBox) {};
\node[
    anchor=south east,
    text=cSplit,
    font=\scriptsize
] at ([yshift=-1.5pt,xshift=3pt]SplitBox.north east)
{\textbf{\textsc{Split}}};

\end{tikzpicture}

\end{algorithm}
\begin{example}
\label{ex:cpr_running}
We consider a PWA system with discrete parameter matrix
\[
E^d=\varnothing,\qquad
P^d=\begin{bmatrix}0&0&-1&1\end{bmatrix},\qquad
B^d=
\begin{bmatrix}
0&0&1&1\\
1&-1&1&1
\end{bmatrix}.
\]
Since $\operatorname{rank}(B^d)=2$, \cref{prop:rat_bounds} gives $k_{\min}=2$.

At iteration $0$, CPR initializes
\[
W_0=
\begin{bmatrix}
0&0&1&1\\
1&-1&1&1
\end{bmatrix},
\]
corresponding to the commitments
$\psi_1=z_3+z_4$ and
$\psi_2=z_1-z_2+z_3+z_4$.
Because $P^d\notin\operatorname{rowspan}(W_0)$, we have
$\widetilde P_0^d=P^d$, and
\[
M_0=
\begin{bmatrix}
0&0&-1&1\\
0&0&1&1\\
1&-1&1&1
\end{bmatrix}.
\]
The submatrix on the first two rows and columns $3,4$ has determinant
\[
\det
\begin{bmatrix}
-1&1\\
1&1
\end{bmatrix}
=-2,
\]
so $M_0$ is not TU.

CPR first applies \textbf{\textsc{Sparsify}} by the row operation
$w_2\leftarrow w_2-w_1$, yielding
\[
W_0^{\mathrm{spa}}=
\begin{bmatrix}
0&0&1&1\\
1&-1&0&0
\end{bmatrix}.
\]
This preserves the row space and therefore introduces no new symbol. However,
$\widetilde P_0^d=P^d$ still holds, and the same determinant-$2$ obstruction remains between
$-z_3+z_4$ and $z_3+z_4$. Thus, re-basing alone cannot remove the failure.

CPR then applies \textbf{\textsc{Split}}. The minimal non-TU core involves $z_3$ and $z_4$; choosing $z_3$, CPR removes its contribution from the existing aggregate and appends the explicit commitment $e_3^\top z=z_3$:
\[
W_1=
\begin{bmatrix}
0&0&0&1\\
1&-1&0&0\\
0&0&1&0
\end{bmatrix}.
\]
The resulting symbols are
\[
\psi_1=z_4,\qquad
\psi_2=z_1-z_2,\qquad
\psi_3=z_3.
\]
Now both the bridge and the discrete inequality are explained by $W_1$:
\[
B^d=
\begin{bmatrix}
1&0&1\\
1&1&1
\end{bmatrix}W_1,
\qquad
P^d=
\begin{bmatrix}
1&0&-1
\end{bmatrix}W_1.
\]
Hence $\widetilde P_1^d=0$, and the verification matrix reduces to
\[
M_1=
\begin{bmatrix}
0&0&0&1\\
1&-1&0&0\\
0&0&1&0
\end{bmatrix},
\]
which is TU. CPR therefore returns $W_\star=W_1$ with $k=3$.

\end{example}

\section{Details of Neural CPR for Network Systems}\label{app:cpr_nn}
Neural CPR follows the same CEGAR structure, but its oracle reads the landscape from samples rather than from a matrix. Let $\mathcal S$ be the set of states drawn from a robot demonstration dataset, and $h_s$ the learned objective at state $s$, minimized over the decision space $\mathcal X$. Neural CPR starts from the empty symbol basis $(\mathcal U_0,\mathcal Q_0)=\varnothing$, so $k=0$ and all of $\mathcal X$ is a single subproblem. For each $s\in\mathcal S$ it builds a sample graph $G_s$ by running noisy multi-start descent on $h_s$ and recording the visited decisions as nodes. Edges join consecutive records along a descent and $m$-nearest neighbours.
Each edge is given the largest objective value at $n_e$ points sampled along its segment ($50\le n_e\le 200$, depending on the task). This value is a lower bound on the maximum of $h_s$ along the segment, so an edge may under-report the barrier that the segment climbs.

At iteration $i$, the current directions $\mathcal U_i$ and quantizers $\mathcal Q_i$ assign
every node in $G_s$ to its subproblem through $(q_j(u_j^\top x))_{j=1}^{k}$. 
Each subproblem $R$ is tested by flooding, a $0$-dimensional sublevel-persistence sweep
\citep{edelsbrunner2002topological} run with union--find: edges are added in increasing level, so each basin forms a pool at its own lowest value $b$ and merges into a lower pool at the level $d$ of the lowest ridge between them. With $a_R$ the lowest value sampled in $R$ and $\tau$ the declared solver resolution, a pool is flagged as violating when $d-b>\tau$, so leaving costs more than the budget, and $b>a_R+\tau$, so settling there already loses more than the budget.
Violating subproblems form $\Phi_i$, and CPR returns when $\Phi_i=\varnothing$. Otherwise it attempts the repair that costs no symbol (\textbf{\textsc{Sparsify}}) before the one that does (\textbf{\textsc{Split}}).

\paragraph{1) \textsc{Sparsify}-Step.}
CPR first asks whether a violation reflects the sampling rather than the symbols. A basin may look sealed only because no sample crossed its escape; CPR then seeds confined descents at the flagged saddle. This operation leave $k$ unchanged. Since added evidence can also expose basins that coarse sampling had hidden, $|\Phi_i|$ is not monotone across iterations.

\paragraph{2) \textsc{Split}-Step.}
Violations that survive are genuine. For each surviving subproblem, CPR takes the lowest pass $\pi$ between two of its violating basins and measures the gradient flip across two short segments through $\pi$, one joining the two basin floors and one along the merging edge; these are sign-aligned and averaged into a normal $u_{k+1}$, whose ridge point $x_\pi$ sets the threshold of the new quantizer $q_{k+1}$ at $u_{k+1}^\top x_\pi$. If $u_{k+1}$ falls within an angle tolerance of an existing direction, only a threshold is added and $k$ is unchanged; otherwise a new direction opens and $k\to k+1$. The loop repeats until no subproblem is violating or a resource cap is reached.

\begin{algorithm}[H]
\footnotesize
\caption{Detailed CPR Algorithm (Network Systems)}
\label{alg:nn_cpr_detailed}
\textbf{Init:} objective family $\{h_s\}_{s\in\mathcal S}$; resolution $\tau$\\
\textbf{Certificate:} $\tau$-descent faithful $W:=(\mathcal U,\mathcal Q)$
\vspace{2pt}\hrule\vspace{4pt}
\begin{algorithmic}[1]

\STATE $(\mathcal U_0,\mathcal Q_0)\gets\varnothing$. \icmt{$k=0$}
\STATE For each $s\in\mathcal S$, build $G_s$ by noisy multi-start descent on $h_s$.
\STATE Set each edge level to the largest value at $n_e$ sampled points.

\REPEAT

\STATE \tikzmarknode{NV1}{
Assign the nodes of $G_s$ to subproblems by $(q_j(u_j^\top x))_{j=1}^{k}$.
}

\STATE \tikzmarknode{NV2}{
Flood each subproblem $R$: add edges in increasing level, merging pools.
}

\STATE \tikzmarknode{NV3}{
Flag a violating pool born at $b$ and merged at $d$ if $d-b>\tau$ and $b>a_R+\tau$.
}

\STATE \tikzmarknode{NV4}{
$\Phi_i\gets$ the subproblems carrying a flagged pool.
}

\STATE
\textbf{If} $\Phi_i=\varnothing$, \textbf{return} $(\mathcal U_i,\mathcal Q_i)$.
\\


\STATE \tikzmarknode{NS1}{
For each $R\in\Phi_i$, add descents at its violating saddle.
}

\STATE \tikzmarknode{NS2}{
Recompute $\Phi_i$. \textbf{If} $\Phi_i=\varnothing$, \textbf{return} $(\mathcal U_i,\mathcal Q_i)$.
}

\STATE \tikzmarknode{NT1}{
\textbf{For} each surviving $R\in\Phi_i$:
}

\STATE \tikzmarknode{NT2}{
\quad $\pi\gets$ lowest pass between two violating basins of $R$.
}

\STATE \tikzmarknode{NT3}{
\quad $(u_{k+1},x_\pi)\gets$ wall normal and ridge point of the flip across $\pi$.
}

\STATE \tikzmarknode{NT4}{
\quad $\mathcal U\gets\mathcal U\cup\{u_{k+1}\}$;\;
$\mathcal Q\gets\mathcal Q\cup\{q_{k+1}\}$ cutting at $u_{k+1}^\top x_\pi$.
}

\STATE \tikzmarknode{NT5}{
$(\mathcal U_{i+1},\mathcal Q_{i+1})\gets(\mathcal U,\mathcal Q)$,\; $i\gets i+1$.
}\icmt{New symbol, $k\rightarrow k+1$}

\UNTIL{$\Phi_i=\varnothing$ or a resource cap is reached}

\STATE \textbf{Return} $W:=(\mathcal U,\mathcal Q)$.

\end{algorithmic}

\begin{tikzpicture}[
    remember picture,
    overlay,
    bx/.style={
        draw=#1,
        densely dotted,
        line width=1pt,
        rounded corners=2pt,
        fill=none,
        inner xsep=1pt,
        inner ysep=0pt
    },
    bx/.default=black!55
]

\node[bx=cVerify,
      fit={(NV1)(NV2)(NV3)(NV4)([xshift=\algw]NV1.west)}] (NVerifyBox) {};
\node[anchor=south east, text=cVerify, font=\scriptsize]
  at ([yshift=-1.5pt,xshift=3pt]NVerifyBox.north east) {\textbf{\textsc{Verify}}};

\node[bx=cSparsify,
      fit={(NS1)(NS2)([xshift=\algw]NS1.west)}] (NSparsifyBox) {};
\node[anchor=south east, text=cSparsify, font=\scriptsize]
  at ([yshift=-1.5pt,xshift=3pt]NSparsifyBox.north east) {\textbf{\textsc{Sparsify}}};

\node[bx=cSplit,
      fit={(NT1)(NT2)(NT3)(NT4)(NT5)([xshift=\algw]NT1.west)}] (NSplitBox) {};
\node[anchor=south east, text=cSplit, font=\scriptsize]
  at ([yshift=-1.5pt,xshift=3pt]NSplitBox.north east) {\textbf{\textsc{Split}}};

\end{tikzpicture}

\end{algorithm}

\section{Further Experimental Results}\label{app:further}
\subsection{Beyond Linear Symbol Axes}\label{app:nonlinear}
\begin{wrapfigure}{r}{0.40\textwidth}
    \vspace{-4mm}
    \centering
    \includegraphics[width=\linewidth]{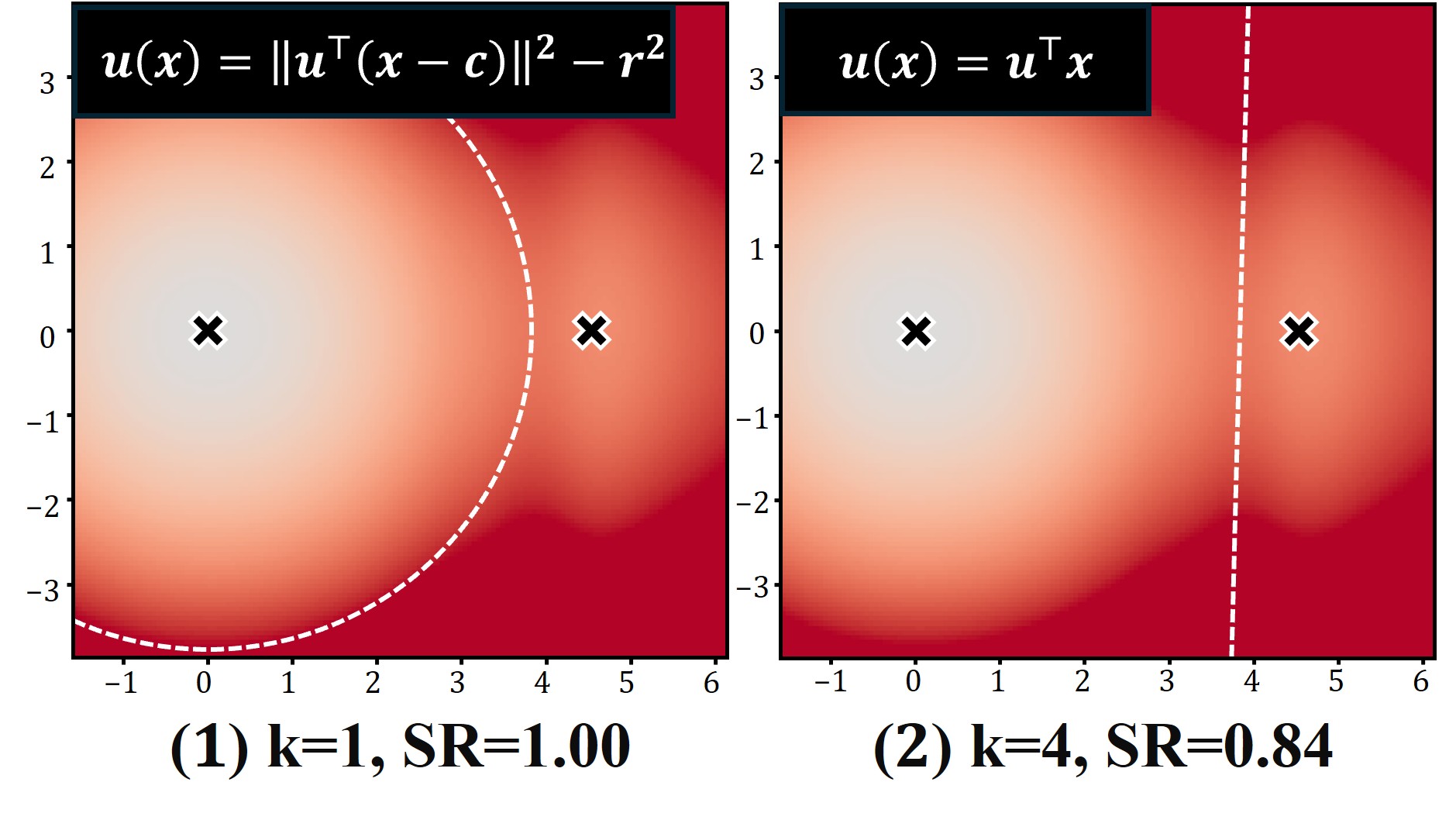}
    \vspace{-20pt}
    \caption{\textbf{Quadratic Symbol Axes.} It captures the curved decision boundary with fewer symbolic directions.}
    \label{fig:nonlinear}
    \vspace{-5pt}
\end{wrapfigure}
Whether linear partitions suffice depends on the feature space in which $x$ is expressed. When it is abstract enough that task-relevant modes are linearly separable, a few cuts $u^\top x$ separate them; when basin boundaries are curved in it, linear cuts can only approximate them piecewise, and more directions are needed. We therefore examine whether richer symbol axes can better capture such nonlinear boundaries. For example, replacing the linear projection $u^\top x$ with a quadratic form can reduce the required symbolic directions while improving success rate. However, increasingly expressive symbol functions transfer complexity from the network into the symbolic layer and weaken interpretability, limiting the value of higher-order representations.

\subsection{Ablations on Solver Resolution $\tau$}
\begin{table*}[t]
\centering

\begin{minipage}[t]{0.57\textwidth}
\vspace{0pt}
\centering

\captionof{table}{\textbf{Ablation of Resolution $\tau$.}
200 descent initializations, 600 optimization steps, $k$NN size 48.}
\label{tab:tau_ablation}

\vspace{-5pt}

{\footnotesize
\renewcommand{\arraystretch}{1.08}
\setlength{\tabcolsep}{3.2pt}

\resizebox{\linewidth}{!}{%
\begin{tabular}{ccccc@{\hspace{5mm}}ccccc}
\toprule
\multicolumn{5}{c}{\texttt{Aligning} ($H=4$)}
&
\multicolumn{5}{c}{\texttt{Avoiding} ($H=8$)}
\\
\cmidrule(lr){1-5}\cmidrule(lr){6-10}

$\tau$ & $k$ & \textbf{Violating} & \textbf{SR} & \textbf{Time (s)}
&
$\tau$ & $k$ & \textbf{Violating} & \textbf{SR} & \textbf{Time (s)}
\\
\midrule

0.05  & 6 & 87\% & 0.66 & 2350
& 0.02 & 3 & 91\% & 0.81 & 1343 \\

0.125 & 6 & 35\% & 0.72 & 1801
& 0.06 & 2 & 0\% & 1.00 & 612 \\

0.25  & 4 & 0\% & 0.84 & 962
& 0.10 & 2 & 0\% & 1.00 & 327 \\

0.50  & 2 & 0\% & 0.81 & 194
& 0.20 & 0 & 0\% & 0.97 & 14 \\

1.00  & 0 & 0\% & 0.94 & 38
& 0.38 & 0 & 0\% & 1.00 & 6 \\

\bottomrule
\end{tabular}%
}
}

\end{minipage}
\hfill
\begin{minipage}[t]{0.405\textwidth}
\vspace{0pt}
\centering

\captionof{table}{\textbf{Ablation of Sampling Evidence.}
\texttt{Avoiding} with $H=16$.}
\label{tab:evidence_ablation}

\vspace{-5pt}

{\footnotesize
\renewcommand{\arraystretch}{1.08}
\setlength{\tabcolsep}{4pt}

\resizebox{\linewidth}{!}{%
\begin{tabular}{ccccc}
\toprule
$\mathbf{k}$\textbf{NN}
& \textbf{Descent Counts}
& $k$
& \textbf{Violating}
& \textbf{Time (s)} \\
\midrule

6  & 200 & 1 & 100\% & 67   \\
12 & 200 & 1 & 100\% & 360  \\
24 & 200 & 2 & 83\%  & 692  \\
48 & 200 & 2 & 67\%  & 1105 \\
96 & 200 & 1 & 32\%  & 1203 \\
96 & 400 & 1 & 0\%   & 1514 \\

\bottomrule
\end{tabular}%
}
}

\end{minipage}

\vspace{-8pt}
\end{table*}
We vary the escape budget $\tau$ and report the discovered symbol dimension $k$, the fraction of conditioned regions that remain \emph{violating} under $\tau$-descent faithfulness, \textbf{SR}, and symbol-discovery time (\cref{tab:tau_ablation}). The results support interpreting $\tau$ as a solver resolution rather than a fitted parameter. When $\tau$ falls below within-basin descent variation, CPR over-resolves noise, producing more symbols while leaving many regions unclosed. In an intermediate range, all regions close and $k$ decreases progressively with little loss in SR. Once $\tau$ exceeds the dominant inter-basin barriers, no explicit symbolic distinction is required ($k=0$). Thus, $\tau$ can be selected from the empirical range between within-basin variation and inter-basin separation.

\subsection{Ablations on Sampling Evidence.}
We next vary the evidence available to neural CPR through graph connectivity and descent initializations (\cref{tab:evidence_ablation}; we take $\tau=0.38$, 600 optimization steps). With insufficient evidence, the inferred symbol dimension $k$ changes with the sampling configuration while many conditioned regions remain violating, indicating that the observed partition is not yet supported by the sampled landscape. Increasing the evidence budget progressively reduces the violating fraction, reaching zero with $k=1$ at 400 initializations. This trend highlights that symbol invention should be interpreted only after the sampled atlas closes: additional evidence resolves uncertain connectivity before the resulting symbolic partition is treated as stable.

\subsection{Symbol Invention Cost}\label{app:invention_cost}
Symbol invention is an offline cost of the same kind as training: it is paid once, and no symbol is invented at deployment. In the PWA setting, $W$ depends only on the coefficient matrices $(E^d,P^d,B^d)$; the start state and the goal enter only the right-hand sides, so CPR runs once per task. It finishes within two seconds on every task (\cref{tab:cpr_time}). In the neural setting, the decision variable is the action chunk, so each planning horizon has its own symbols and neural CPR runs once per horizon; its running time is reported in \cref{tab:tau_ablation,tab:evidence_ablation}.

\begin{table}[h]
\centering\small
\caption{\textbf{CPR Time on PWA Tasks.} $W$ is computed once per task.}
\label{tab:cpr_time}
\vspace{-5pt}
\renewcommand{\arraystretch}{1.08}
\begin{tabular}{lcccccc}
\toprule
& \texttt{MPushT} & \texttt{Hook} & \texttt{Shelf} & \texttt{Block} & \texttt{Cuboid} & \texttt{Pen} \\
\midrule
\textbf{time (s)} & $0.43$ & $1.66$ & $0.48$ & $0.11$ & $0.11$ & $0.18$ \\
\bottomrule
\end{tabular}
\end{table}

The online metrics in \cref{tab:pwa_main,tab:nn_main} exclude offline costs for every method, not only for ours. For instance, the classifier training of CoCo and the offline learning of the symbol-invention baselines are excluded in the same way. For reference, training the classifier of CoCo takes about $6.14\times$ as long as CPR on average. CPR itself needs neither a demonstration dataset nor training.

\section{Implementation Details}\label{app:details}
\subsection{Task Details}
We evaluate six tasks (three 2D manipulation tasks, three in-hand reorientation tasks) with PWA objectives and five tasks (three D3IL tasks  \citep{jia2024towards}, two real-world tasks) with neural network settings. Task details can be found as follows.
\begin{itemize}
    \item \texttt{MPushT}.
    4 pushers translate and rotate a planar T-shaped object. There are a finite number of pushable points on each face. A rollout succeeds when its position and orientation errors are below $0.15$ and $0.15\,\mathrm{rad}$, respectively.
    \item \texttt{Hook}.
    A planar gripper manipulates a finite-volume L-shaped hook to reach around and pull a $0.318$-wide square block toward a divider. The task requires collision-free repositioning, establishing a valid hook--block contact, and sustained pulling.
    \item \texttt{Packing}.
    The controller packs $4$--$8$ randomly selected, heterogeneous rectangles into two stable support stacks inside a one-sided shelf with a $1.70\times1.31$ interior. Objects begin with randomized orientations and must respect spatial supports and constraints. 
    \item \texttt{Block}.
    A five-finger Shadow Hand reorients an $80\,\mathrm{mm}$ cube while maintaining a feasible multi-finger grasp in the IsaacLab simulator. Its contact structure distinguishes five fingers and six object faces. Terminate when the rotational error is at most $0.10\,\mathrm{rad}$.
    \item \texttt{Cuboid}.
    The Shadow Hand reorients a $55\times70\times110\,\mathrm{mm}$ cuboid under the same five-finger setting. Unlike the symmetric cube, its unequal dimensions make fingertip reachability and face transitions orientation dependent. Take $0.10\,\mathrm{rad}$ rotational-error tolerance.
    \item \texttt{Pen}.
    The hand reorients a capsule-shaped pen of length $200\,\mathrm{mm}$ and diameter $16\,\mathrm{mm}$. Contacts are organized across three finger roles and three axial pen regions, capturing both grasp redistribution and longitudinal contact changes. Take $0.10\,\mathrm{rad}$ success tolerance.
    \item \texttt{Avoiding}.
    A planar robot endpoint traverses three obstacle levels containing one, two, and three obstacles, separated vertically by $0.18\,\mathrm{m}$. The corresponding $2$, $3$, and $4$ passage choices induce $24$ collision-free route modes. Success requires crossing the goal line at $y=0.35\,\mathrm{m}$ without collision, within a $250$-step episode.
    \item \texttt{Pushing}.
    A robot pushes two $30\,\mathrm{mm}$ boxes into two $50\,\mathrm{mm}$ target regions. Either color-preserving or crossed assignment is accepted when both box--target center distances are at most $50\,\mathrm{mm}$. Combining the two assignments with their two completion orders produces four distinct solution modes.
    \item \texttt{Aligning}.
    A robot pushes a planar object to a randomized target pose from either its inner or outer side (two task modalities). Initial and target orientations span $[-90^\circ,90^\circ]$. Success requires position and orientation errors below $18\,\mathrm{mm}$ and $0.048\pi\,\mathrm{rad}$, respectively.
    \item \texttt{Matching}. 
    A robot places multiple red and blue blocks of varying shapes into fixed color-matched baskets in arbitrary order. Success requires all blocks to be placed in the basket of the corresponding color.
    \item \texttt{Sorting}.
    A robot inserts shape-specific blocks (e.g., cross and circle) into a fixed storage box through their matching apertures. Each block fits only its corresponding opening, requiring precise alignment and insertion.
\end{itemize}

\subsection{Energy Functions of the Evaluated Policies}\label{app:energy}

We view a trained policy as a conditional distribution $\pi(x\mid s)$ over the decision $x$ and define
\[
    h_s(x) \;=\; -\log \pi(x\mid s) + c(s),
\]
where $c(s)$ does not depend on $x$ and hence changes neither the minimizers, the gradients, nor the barrier heights. When $-\log\pi$ has no closed form, we use a surrogate of it. In our paper, each baseline policy is instantiated as follows.

\begin{itemize}
    \item \textbf{DDPM:} The surrogate function is the denoising loss at $x$ on $n$ fixed pairs of diffusion time and noise $(t_i,\epsilon_i)$:
    \[
        h_s(x)=\frac{1}{n}\sum_{i=1}^{n}
        \big\lVert \epsilon_\theta(x_{t_i},t_i,s)-\epsilon_i \big\rVert^2,
        \qquad
        x_{t_i}=\sqrt{\bar\alpha_{t_i}}\,x+\sqrt{1-\bar\alpha_{t_i}}\,\epsilon_i .
    \]
    This is the DDPM variational bound without its time weights on a fixed set of times: a surrogate of $-\log\pi$, not a bound on it. 
    \item \textbf{BESO:} Also a diffusion policy; the same construction with its own denoising loss, $h_s(x)=\frac{1}{n}\sum_{i}\lVert D_\theta(x+\sigma_i\epsilon_i;\sigma_i,s)-x\rVert^2$, on fixed pairs of noise level and noise $(\sigma_i,\epsilon_i)$.
    \item \textbf{IBC:} An explicit energy model, $\pi\propto\exp(-E_\theta)$, so $h_s(x)=E_\theta(s,x)$, with the partition function absorbed into $c(s)$.
    \item \textbf{ACT, BeT:} $h_s(x)=E_\phi(s,x)$, an energy network fitted to actions sampled from the frozen policy; it is a learned surrogate of $-\log\pi(x\mid s)$.
\end{itemize}

Given the weights of a trained policy, we obtain $h_s$ in one of two ways. Most baselines (DDPM, BESO, IBC) adopt a \textit{read-out} method, we write $-\log\pi(x\mid s)$, or its variational surrogate, as a function of $x$ and fix every source of randomness in it once, so that $h_s$ is deterministic; no training is involved. For the other baselines such as ACT and BeT, we use a \textit{fitting} method for a better performance. We sample actions from the frozen policy and train an energy network $E_\phi(s,x)$ on them with InfoNCE. In both cases, values are computed by forward passes and gradients by automatic differentiation with respect to $x$, and all policies share the same solver and budget.

\subsection{MILP Setups}
\begin{figure*}[t]
\centering

\begin{minipage}[t]{0.49\textwidth}
    \centering
    \includegraphics[width=0.58\linewidth]{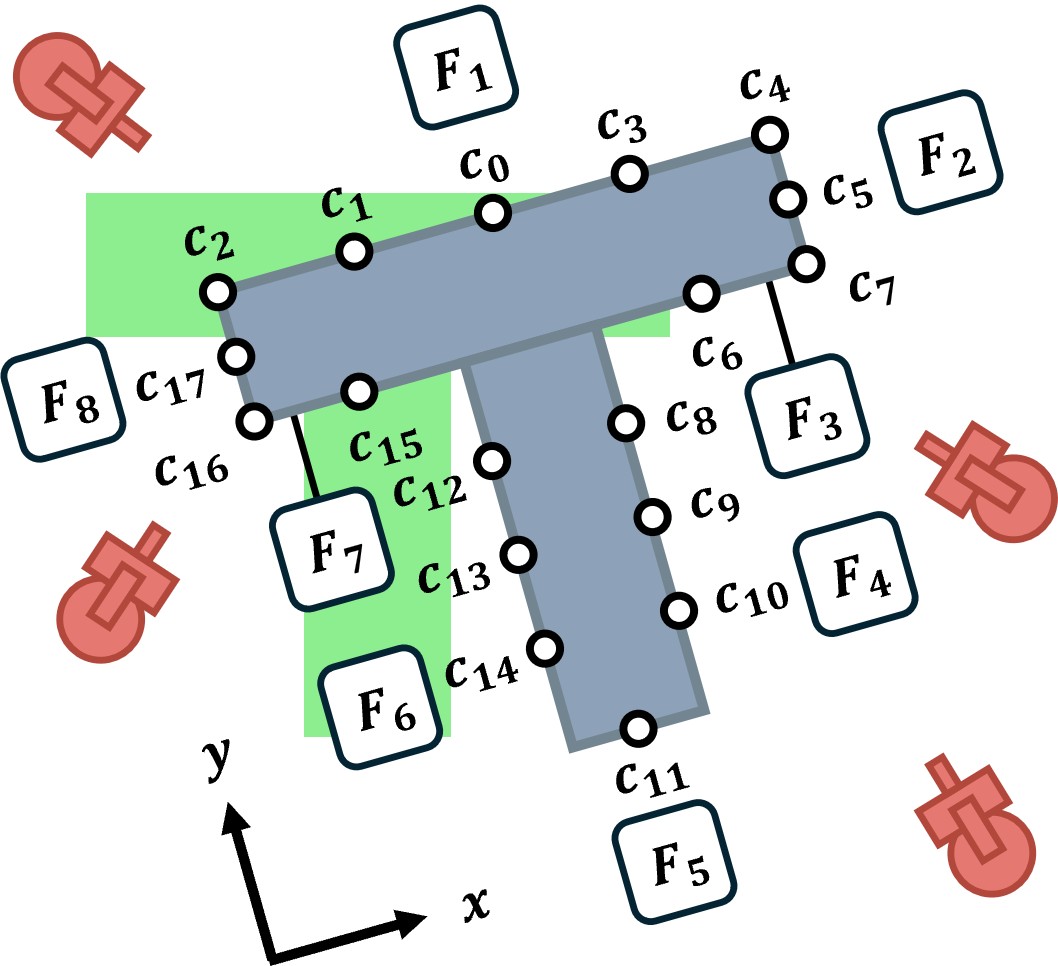}
    \caption{\texttt{MPushT}\textbf{ Contact Domain.} The object boundary is discretized into faces $F_i$ and candidate contact points $c_j$; the local $x$--$y$ axes define the object-fixed coordinate frame .}
    \label{fig:mpusht_contact}
\end{minipage}
\hfill
\begin{minipage}[t]{0.49\textwidth}
    \centering
    \includegraphics[width=0.58\linewidth]{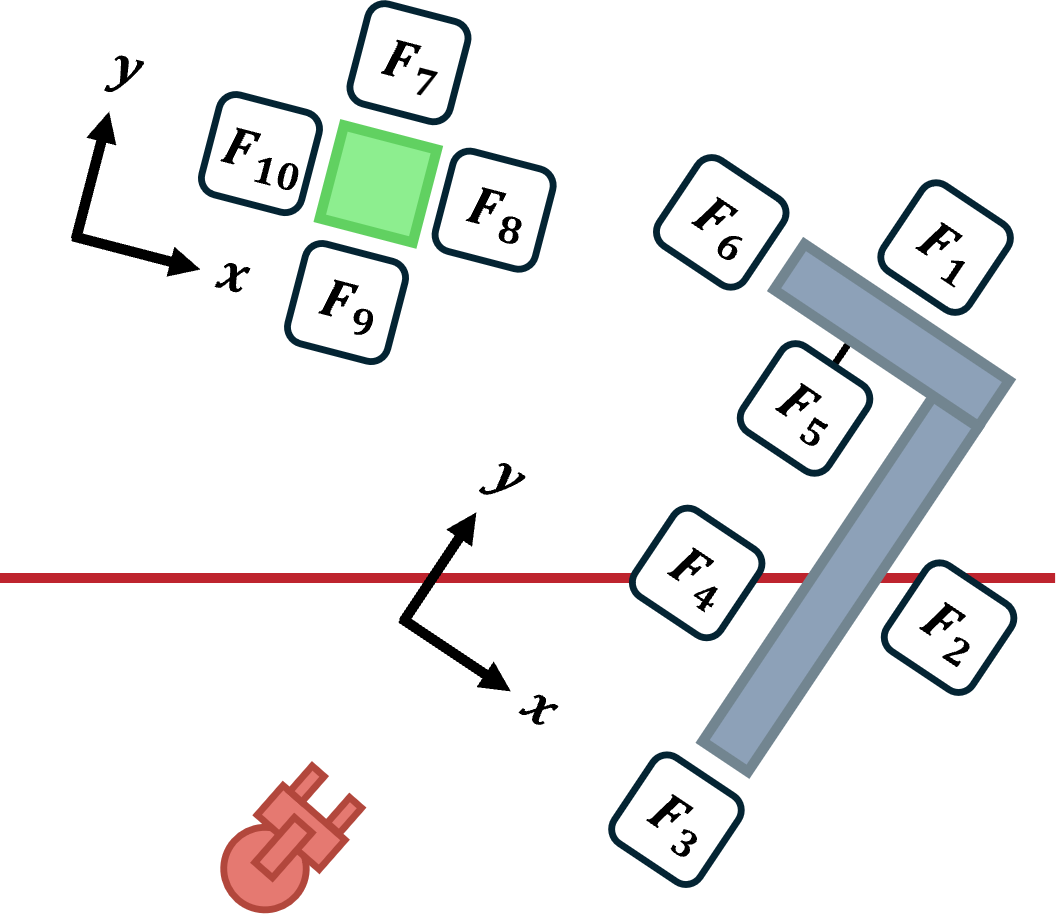}
    \caption{\texttt{Hook}\textbf{ Contact Domain.} The hook and object are represented by multiple discrete contact faces $F_i$ relative to their local $x$--$y$ coordinate frames.}
    \label{fig:hook_contact}
\end{minipage}
\end{figure*}

In this subsection, we summarize the MILP structure through its primitive discrete variables $z$, which define the mode and contact choices. We omit the continuous dynamics and full constraint matrices for space. We first outline the rules upon which our formulation is based, and then demonstrate the specific meaning of $z$ (\cref{tab:milp_z_2d} and \cref{tab:milp_z_inhand}).
\begin{itemize}
    \item{\texttt{MPushT}.}
    A T-shaped object is pushed to a goal pose through 18 candidate contact points distributed over eight faces. Each binary factor $z_j$ activates a unit push at contact $c_j$ along the inward face normal. Contacts on the same face are activated cumulatively to form contiguous pushing segments, with at most four faces active at each step (\cref{fig:mpusht_contact}).
    
    \item{\texttt{Hook}.}
    A gripper manipulates a hook to pull a block toward a divider wall. The 16 binary factors encode whether the hook is held, hook--block contact, their relative placement and facing surfaces, and whether the block can reach the wall within the current step (\cref{fig:hook_contact}).
    
    \item{\texttt{Shelf}.}
    Multiple objects are packed onto a shelf through discrete pick-and-place primitives and we take $8$ objects here as an example. The 41 factors consist of eight grasp actions, placements on the shelf floor, and stacking actions among objects. An object must be placed only on a larger supporting object. 
    
    \item{\texttt{Block}.}
    A five-finger hand reorients a cube through a discrete finger--face contact graph. We denote that, ff: index finger; mf: middle finger; rf: ring finger; lf: little finger; th: thumb. The 20 factors represent active fingers, admissible finger--face assignments, contacted object faces, and one idle unit. Flow constraints allow multiple simultaneous contacts, with each active finger assigned to one face and no face contacted twice.
    
    \item{\texttt{Cuboid}.}
    The cuboid task uses the same 20-factor finger--face representation as \texttt{Block}. The unequal object dimensions change the continuous contact feasibility and reachable orientations, while the discrete factors retain the same finger, face, and finger--face semantics.
    
    \item{\texttt{Pen}.}
    A five-finger hand reorients a pen using three functional finger groups---opposition, medial, and ulnar---and three axial contact zones---proximal, center, and distal. The 12 factors encode active groups, group--zone assignments, and occupied zones; two groups contact two distinct zones at each step.
\end{itemize}

\begin{table}[t]
\centering
\caption{\textbf{Discrete Factors for 2D Manipulation.}}
\vspace{-5pt}
\label{tab:milp_z_2d}
\small
\renewcommand{\arraystretch}{1.08}
\setlength{\tabcolsep}{4pt}
\begin{tabular}{p{0.09\linewidth} p{0.08\linewidth} p{0.75\linewidth}}
\toprule
\textbf{Task} & \textbf{Factors} & \textbf{Semantics} \\
\midrule

\multirow{8}{*}{\texttt{MPushT}}
& $z_{0:4}$   & $F_1$, push along $-y$; cumulative runs $c_0\!\rightarrow\! c_4$ \\
& $z_5$       & $F_2$, push along $-x$ at $c_5$ \\
& $z_{6:7}$   & $F_3$, push along $+y$ from $c_6$ to $c_7$ \\
& $z_{8:10}$  & $F_4$, push along $-x$ at $c_8$ and along $c_9\!\rightarrow\!c_{10}$ \\
& $z_{11}$    & $F_5$, push along $+y$ at $c_{11}$ \\
& $z_{12:14}$ & $F_6$, push along $+x$ from $c_{12}$ to $c_{14}$ \\
& $z_{15:16}$ & $F_7$, push along $+y$ from $c_{15}$ to $c_{16}$ \\
& $z_{17}$    & $F_8$, push along $+x$ at $c_{17}$ \\
\midrule

\multirow{5}{*}{\texttt{Hook}}
& $z_{0:1}$   & hook released / held by the gripper \\
& $z_{2:3}$   & hook separate from / in contact with the block \\
& $z_{4:9}$   & relative block--hook placement: body-frame axis, side, and active hook surface \\
& $z_{10:13}$ & block face oriented toward the hook: body-frame axis and sign \\
& $z_{14:15}$ & block cannot / can reach the divider wall within the current step \\
\midrule

\multirow{3}{*}{\texttt{Shelf}}
& $z_{0:7}$
& grasp object $o$ \\
& $z_{8:15}$
& place $o$ on the shelf floor \\
& $z_{16:40}$
& place $o$ on $o'$; permitted only when $o$ is smaller than $o'$ \\
\bottomrule
\end{tabular}
\end{table}

\begin{table}[t]
\centering
\caption{\textbf{Discrete Factors for In-Hand Reorientation.}}
\vspace{-5pt}
\label{tab:milp_z_inhand}
\small
\renewcommand{\arraystretch}{1.08}
\setlength{\tabcolsep}{4pt}
\begin{tabular}{p{0.09\linewidth} p{0.08\linewidth} p{0.75\linewidth}}
\toprule
\textbf{Task} & \textbf{Factors} & \textbf{Semantics} \\
\midrule

\multirow{4}{*}{\texttt{Block}}
& $z_{0:4}$
& active finger: ff, mf, rf, lf, or th \\
& $z_{5:12}$
& admissible finger--face assignment \\
& $z_{13:18}$
& contacted face: palm-facing, dorsal, thumb-side, little-side, distal, or proximal \\
& $z_{19}$
& idle unit, allowing only two active finger contacts \\
\midrule

\texttt{Cuboid}
& $z_{0:19}$
& same finger, finger--face, face, and idle-factor semantics as \texttt{Block}; contact feasibility follows the cuboid geometry \\
\midrule

\multirow{3}{*}{\texttt{Pen}}
& $z_{0:2}$
& active finger group: opposition, medial, or ulnar \\
& $z_{3:8}$
& finger-group assignment to proximal, center, or distal pen regions \\
& $z_{9:11}$
& occupied pen zone: proximal, center, or distal \\
\bottomrule
\end{tabular}
\end{table}

\subsection{\texttt{MPushT} MILP Variations}\label{app:mpusht_variation}
To study why certain human-defined concepts are useful, we further construct nine expert-designed MILP formulations of the same \texttt{MPushT} task, containing 127 primitive contact factors in total. Holding the task fixed while varying its discrete factorization allows us to analyze the computational role of different human-defined distinctions. We summarize the factors below.

\begingroup
\footnotesize
\renewcommand{\arraystretch}{1.08}
\setlength{\tabcolsep}{3.5pt}
\setlength\LTleft{\fill}
\setlength\LTright{\fill}

\begin{longtable}{
>{\centering\arraybackslash}p{0.09\linewidth}
p{0.10\linewidth}
p{0.73\linewidth}}

\caption{\textbf{Alternative Expert-Designed MPushT Factorizations.} $P_0$ is described in \cref{tab:milp_z_2d}.}
\label{tab:mpsuht_variations}\\

\toprule
\textbf{Form.} & \textbf{Factors} & \textbf{Semantics} \\
\midrule
\endfirsthead

\toprule
\textbf{Form.} & \textbf{Factors} & \textbf{Semantics} \\
\midrule
\endhead

\midrule
\multicolumn{3}{r}{\textit{(Continued on next page)}} \\
\endfoot

\bottomrule
\endlastfoot

\multirow{4}{*}{$P_1$}
& $z_{0:2}$ &
$F_1$: push at $c_0$, extend leftward to $c_1$, or move to $c_3$ \\*
& $z_{3,4,7,10,11}$ &
single pushes on $F_2,F_3,F_5,F_7,F_8$ at $c_5,c_6,c_{11},c_{15},c_{17}$ \\*
& $z_{5:6}$ &
$F_4$: push at $c_8$ or move to $c_9$ \\*
& $z_{8:9}$ &
$F_6$: push at $c_{12}$ or extend to $c_{13}$ \\
\midrule

$P_2$
& $z_{0:7}$ &
one occupancy bit per face $F_1,\ldots,F_8$, pushing at its innermost contact point \\
\midrule

\multirow{2}{*}{$P_3$}
& $z_{0:3}$ &
push from the top, bottom, right, or left side \\*
& $z_{4:5}$ &
global counter-clockwise / clockwise spin; the first active side shifts its contact to realize the requested rotation \\
\midrule

\multirow{2}{*}{$P_4$}
& $z_{0:3}$ &
single pushes from the top, bottom, right, or left \\*
& $z_{4:7}$ &
two-point couples for counter-clockwise ($z_{4:5}$) or clockwise ($z_{6:7}$) rotation\\
\midrule

\multirow{2}{*}{$P_5$}
& $z_0$ &
no push \\*
& $z_{1:26}$ &
one-hot unit-effect targets $(f_x,f_y,s)\in\{-1,0,+1\}^3\setminus\{(0,0,0)\}$, decoded to the fewest $P_0$ contacts realizing the target \\
\midrule

$P_6$
& $z_{0:17}$ &
same contact-point factors as $P_0$, but with at most two active faces per step \\
\midrule

\multirow{2}{*}{$P_7$}
& $z_{0:7}$ &
same face-occupancy factors as $P_2$ \\*
& $z_{8:17}$ &
face-specific lever offsets that move or extend the active contact along $F_1,F_3,F_4,F_6,$ or $F_7$ \\
\midrule

\multirow{2}{*}{$P_8$}
& $z_{0:7}$ &
same face-occupancy factors as $P_2$ \\*
& $z_{8:11}$ &
one / two units of global counter-clockwise ($z_{8:9}$) or clockwise ($z_{10:11}$) spin, realized by shifting the first eligible active face \\

\end{longtable}
\endgroup

\end{document}